\documentclass[11pt]{article}

\usepackage{amsthm,amsmath,amssymb,bbm}
\usepackage{natbib}
\usepackage{multirow}
\usepackage[pdftex]{graphicx}
\usepackage{subfigure}
\usepackage{makecell}
\usepackage{booktabs}
\usepackage{array}
\usepackage{url}
\usepackage{algorithm}
\usepackage{algorithmic}
\usepackage{bm}
\usepackage{wrapfig}
\usepackage{lipsum}
\usepackage{mathrsfs}
\usepackage{dsfont}
\usepackage{titling}
\usepackage{multirow}

\usepackage[usenames,dvipsnames,svgnames,table]{xcolor}
\usepackage[colorlinks,
linkcolor=red,
anchorcolor=blue,
citecolor=blue
]{hyperref}

\usepackage{smile}

\newcommand*{\supp}{\mathrm{supp}}

\newcommand \ru{\mathrm{u}}

\newcommand{\nn}{\nonumber}

\newcommand{\rF}{{\rm F}}
\newcommand{\reff}{{\rm eff}}

\def\##1\#{\begin{align}#1\end{align}}
\def\$#1\${\begin{align*}#1\end{align*}}

\newcommand {\vecc}{\textnormal {vec}}

\def\T{{ \mathrm{\scriptscriptstyle T} }}

\newcommand{\Rom}[1]{\text{\uppercase\expandafter{\romannumeral #1\relax}}}

\usepackage{geometry}
\usepackage{enumitem}

\begin{document}

\title{\LARGE Robust Sparse Reduced Rank Regression in High Dimensions}
\author{Kean Ming Tan, Qiang Sun~and~Daniela Witten}

\date{\today}

\maketitle

\begin{abstract}
We propose robust sparse reduced rank regression for analyzing large and complex high-dimensional data with heavy-tailed random noise.  
The proposed method is based on a convex relaxation of  a rank- and sparsity-constrained non-convex optimization problem, which is then solved using the alternating direction method of multipliers algorithm. 
We establish non-asymptotic estimation error bounds under both Frobenius and nuclear norms in the high-dimensional setting.
This is a major contribution over existing results in reduced rank regression, which mainly focus on rank selection and prediction consistency.  
 Our theoretical results quantify the tradeoff between heavy-tailedness of the random noise and statistical bias.  
 For random noise with bounded $(1+\delta)$th moment with $\delta \in (0,1)$, the rate of convergence is a function of $\delta$, and is slower than the sub-Gaussian-type deviation bounds; for random noise with bounded second moment, we obtain a rate of convergence as if sub-Gaussian noise were assumed.   Furthermore, the transition between the two regimes is smooth. 
We illustrate the performance of the proposed method via extensive numerical studies and a data application.

\end{abstract}


\section{Introduction}
\label{sec:introduction}
Low rank matrix approximation methods have enjoyed successes in modeling and extracting information from  large and complex data across various scientific disciplines. 
However, large-scale data sets are often accompanied by outliers due to possible measurement error, or because the population exhibits a leptokurtic distribution. 
As shown in \citet{she2017robust}, one single outlier can have a devastating effect on low rank matrix estimation. 
Consequently, non-robust procedures for low rank matrix estimation could lead to inferior estimates and spurious scientific conclusions. 
For instance, in the context of financial data, it is evident that asset prices follow heavy-tailed distributions: if the heavy-tailedness is not accounted for in statistically modeling, then the recovery of common market behaviors and asset return forecasting may be jeopardized \citep{cont2001empirical,muller1998heavy}.

In the context of reduced rank regression, \cite{she2017robust} addressed this challenge by explicitly modeling the outliers with a sparse mean shift matrix of parameters.
This approach requires an augmentation of the parameter space, which introduces a new statistical challenge:  it raises possible identifiability issues between the parameters of interest and the mean shift  parameters.   
For instance, \cite{robustpca} proposed a form of robust principal component analysis by introducing an additional  sparse matrix to model the outliers.   To ensure identifiability, an incoherence condition is assumed on the singular vectors of the original parameter of interest. In other words, the parameter of interest cannot be sparse.  
Therefore, it is unclear whether \citet{she2017robust} can be generalized to the high-dimensional setting in which the number of covariates is larger than the number of observations.
Similar ideas have been considered in the context of robust linear regression  \citep{she2011outlier} and robust clustering  \citep{rcc2016,rss2016}.   

In many statistical applications, the outliers themselves are not of interest. 
Rather than introducing additional parameters to model the outliers, it is more natural to develop robust statistical methods that are less sensitive to outliers.  
There is limited work along these lines in low rank matrix approximation problems. 
In fact, \cite{she2017robust} pointed out that in the context of reduced rank regression, directly applying a robust loss function that down-weights the outliers, such as the Huber loss, may result in nontrivial computational and theoretical challenges due to the low rank constraint. 
So a natural question arises:  can we develop a computationally efficient robust sparse low rank matrix approximation procedure that is less sensitive to outliers and yet has sound statistical guarantees?

In this paper, we propose a novel method for fitting robust sparse reduced rank regression in the high-dimensional setting.  We propose to minimize the Huber loss function subject to both sparsity and rank constraints. This leads to a non-convex optimization problem, and is thus computational intractable. To address this challenge, we consider a convex relaxation, which can be solved via an alternating direction method of multipliers algorithm.  
Most of the existing theoretical analysis of reduced rank regression focuses on rank selection consistency and prediction consistency \citep{bunea2011optimal,mukherjee2011reduced,bunea2012joint,chen2013reduced}. Moreover, the theoretical results for robust reduced rank regression of \citet{she2017robust} are developed under the assumption that the design matrix is low rank.  Non-asymptotic analysis of the estimation error, however, is not well-studied in the context of reduced rank regression, especially in the high-dimensional setting. 
To bridge this gap in the literature, we provide non-asymptotic analysis of the estimation error under  both Frobenius and nuclear norms for robust sparse reduced rank regression.  
Our results require a matrix-type restricted  eigenvalue condition, and are free of incoherence conditions that arise from the identifiability issues discussed in \citet{robustpca}.

The robustness of our proposed estimator is evidenced by its finite sample performance in the presence of heavy-tailed data, i.e., data for which high-order moments are not finite. When the sampling distribution is heavy-tailed, there is a higher chance that some data are sampled far away from their mean. We refer to these outlying data as heavy-tailed outliers. 
Theoretically, we establish non-asymptotic results that quantify the tradeoff between heavy-tailedness of the random noise and statistical bias: for random noise with bounded $(1+\delta)$th moment, the rate of convergence, depending on $\delta$, is slower than the sub-Gaussian-type deviation bounds; for random noise with bounded second moment, we recover results as if sub-Gaussian errors were assumed; and the transition between the two regimes is smooth.

The Huber loss has a \emph{robustification  parameter} that trades bias for robustness. 
In past work, the robustification parameter is usually fixed using the $95\%$-efficiency rule (among others, \citealp{huber1964robust, huber1973robust,portnoy1985asymptotic,mammen1989asymptotics,he1996general}).   
Therefore, estimators obtained under Huber loss are typically biased. 
To achieve asymptotic unbiasedness and robustness simultaneously, within the context of robust linear regression, \cite{sun2016adaptive} showed that the robustification parameter has to adapt to the sample size, dimensionality, and moments of the random noise.
Motivated by \cite{sun2016adaptive}, we will establish theoretical results for the proposed   method by allowing the robustification parameter to diverge.

Heavy-tailed robustness is different from the conventional perspective on robust statistics under the Huber's $\epsilon$-contamination model, which  focuses on developing robust procedures with a high breakdown point \citep{huber1964robust}. 
The breakdown point of an estimator is defined roughly as  the proportion of arbitrary outliers an estimator can tolerate before the estimator produces arbitrarily large estimates, or breaks down \citep{hampel1971general}. 
Since the seminal work of \cite{tukey1975mathematics}, a number of depth-based procedures have been proposed for this purpose (among others, \citealp{liu1990notion, zuo2000general, mizera2002depth,salibian2002bootrapping}). 
Other research directions for robust statistics focus on robust and resistant $M$-estimators: these include the least median of squares and least trimmed squares \citep{rousseeuw1984least}, the S-estimator  \citep{rousseeuw1984robust}, and the MM-estimator \citep{yohai1987high}.   We refer to \cite{portnoy2000robust} for  a literature review on classical robust statistics, and \cite{chen2018robust} for recent developments on non-asymptotic analysis  under the $\epsilon$-contamination model.\\

\noindent \textbf{Notation:} 
For any  vector $\ub = (\ru_1, \ldots, \ru_p)^\T \in \RR^p$ and $q \geq 1$, let $\|\ub\|_q=\big(\sum_{j=1}^p |\ru_j|^q\big)^{1/q}$ denote the $\ell_q$ norm.  
Let $\|\ub\|_0 = \sum_{j=1}^p 1(\ru_j \!\neq\! 0 )$ denote the number of nonzero entries of $\ub$, and let $\|\ub\|_\infty=\max_{1\leq j\leq p}|\ru_j|$. 
For any two vectors $\ub , \vb \in \RR^p$,  let $\langle \ub, \vb \rangle = \ub^\T \vb$. Moreover,  for two sequences of real numbers $\{ a_n \}_{n\geq 1}$ and $\{ b_n \}_{n\geq 1}$, $a_n \lesssim b_n$ signifies that $a_n \leq C b_n$ for some constant $C>0$ that is  independent of $n$, $a_n \gtrsim b_n$ if $b_n \lesssim a_n$, and $a_n \asymp b_n$ signifies that $a_n \lesssim b_n$ and $b_n \lesssim a_n$. If $\Ab$ is an $m\times n$ matrix, we use $\| \Ab \|_q$ to denote its order-$q$ operator norm, defined by $\| \Ab \|_q = \max_{ \ub \in \RR^n} \| \Ab \ub \|_q/\|\ub\|_q$. We define the $(p,q)$-norm of a $m\times n$ matrix $\Ab$ as the usual $\ell_q$ norm of the vector of row-wise $\ell_p$ norms  of $\Ab$: $\big\|\Ab\big\|_{p,q}\equiv \big\|\big(\|\Ab_{1\cdot}\|_p, \ldots, \|\Ab_{m\cdot}\|_{p})\big\|_q$, where  $\Ab_{j\cdot}$ is the $j$th row of $\Ab$. 
We use $\|\Ab\|_* = \sum_{k=1}^{\min\{m,n\}} \lambda_{k}$ to denote the nuclear norm of $\Ab$, where $\lambda_k$ is the $k$th singular value of $\Ab$.
Let $\|\Ab\|_\rF =  \sqrt{\sum_{i=1}^m\sum_{j=1}^n A_{ij}^2}$ be the Frobenius norm of $\Ab$. 
Finally, let $\vecc(\Ab)$  be the vectorization of the matrix $\Ab$, obtained by concatenating the columns of $\Ab$ into a vector. 

\section{Robust Sparse Reduced Rank Regression}
\label{sec:method} 
\subsection{Formulation}
Suppose we observe $n$ independent samples of $q$-dimensional response variables and $p$-dimensional covariates.
Let $\Yb \in \RR^{n\times q}$ be the observed response and let $\Xb\in \RR^{n\times p}$ be the observed covariates. 
We consider the matrix regression model 
\begin{equation}\label{eq:main}
\Yb = \Xb \Ab^*+ \Eb,
\end{equation}
where $\Ab^* \in \RR^{p\times q}$ is the underlying regression coefficient matrix and  
$\Eb\in \RR^{n\times q}$ is an error matrix.  Each row of $\Eb$ is an independent mean-zero and potentially heavy-tailed random noise vector. 

Reduced rank regression seeks to characterize the relationships between $\Yb$ and $\Xb$ in a parsimonious way by restricting  the rank of $\Ab^*$ \citep{izenman1975reduced}.  
An estimator of $\Ab^*$ can be obtained by solving the  optimization problem
\begin{equation}
\label{eq:optimization1}
\underset{\Ab\in \RR^{p\times q}}{\mathrm{minimize}}~  \mathrm{tr} \left\{ (\Yb-\Xb \Ab)^\T (\Yb-\Xb \Ab) \right\}, \qquad \mathrm{subject~to~} \mathrm{rank}(\Ab) \le r,
\end{equation}
where $r$ is typically much smaller than $\min\{n, p,q\}$. 
Due to the rank constraint on $\Ab$, \eqref{eq:optimization1} is non-convex: nonetheless, the global solution of \eqref{eq:optimization1} has a closed form solution \citep{izenman1975reduced}.

It is well-known that squared error loss is sensitive to outliers or heavy-tailed random error  \citep{huber1973robust}.  To address this issue, it is natural to substitute the squared error loss with a loss function that is robust against outliers.  We propose to estimate $\Ab^*$ under the Huber loss function, formally defined as follows.
\begin{definition}[{\sf Huber Loss and Robustification Parameter}] \label{Huber.def}
The Huber loss $\ell_\tau(\cdot)$  is defined as
$$
	\ell_\tau(z) =
	\left\{\begin{array}{ll}
	\frac{1}{2}z^2 ,    & \mbox{if } |z | \leq \tau ,  \\
	\tau |z | - \frac{1}{2} \tau^2,   &  \mbox{if }  |z | > \tau ,
	\end{array}  \right. 	
$$	
where $\tau>0$ is referred to as the {\it robustification parameter} that trades bias for robustness.
\end{definition}

 The Huber loss function blends the squared error loss ($|z| \le \tau$) and the absolute deviation loss ($|z|>\tau$), as determined by  the robustification parameter $\tau$.  
Compared to the squared error loss, large values of $z$ are down-weighted under the  Huber loss, thereby resulting in robustness.  Generally,  an estimator obtained from minimizing the Huber loss is biased.
The robustification parameter $\tau$ quantifies the tradeoff between bias and robustness: a smaller value of $\tau$ introduces more bias but also encourages the estimator to be more robust to outliers.  
We will provide guidelines for selecting $\tau$ based on the sample size and the dimensions of $\Ab^*$ in later sections.  
Throughout the paper, for $\Mb\in \RR^{p\times q}$, we write  $\ell_\tau (\Mb)=\sum_{i=1}^p\sum_{j=1}^q \ell_\tau(M_{ij})$ for notational convenience.

In the high-dimensional setting in which $n < p$ or $n < q$, it is theoretically challenging to estimate $\Ab^*$ accurately without imposing additional structural assumptions in addition to the low rank assumption.  
To address this challenge, \citet{chen2012reduced} and \citet{chen2012sparse} proposed methods for simultaneous dimension reduction and variable selection. 
In particular, they decomposed $\Ab^*$ into the product of its singular vectors, and imposed sparsity-inducing penalty on the left and right singular vectors. 
Thus, their proposed methods involve solving optimization problems with non-convex objective.

Given that the goal is to estimate $\Ab^*$  rather than its singular vectors, we propose to estimate $\Ab^*$ directly.  
Under the Huber loss, a robust and sparse estimate of $\Ab^*$ can be obtained by solving the  optimization problem:
\begin{equation}
\label{eq:opt:conv-non}
\underset{\Ab\in \RR^{p\times q}}{\mathrm{minimize}} \;\bigg\{{\frac{1}{n}\ell_\tau  \left(\Yb-\Xb \Ab \right)}\bigg\},
\qquad \mathrm{subject~to~} \mathrm{rank}(\Ab) \le r\quad \mathrm{and} \quad \mathrm{card}(\Ab)\le k,
\end{equation}
where $\mathrm{card}(\Ab)$ is the number of non-zero elements in $\Ab$.
Optimization problem \eqref{eq:opt:conv-non} is non-convex due to the rank and cardinality constraints on $\Ab$.
We instead propose to estimate $\Ab^*$ by solving the following convex relaxation:
\begin{equation}
\label{eq:opt:conv}
\underset{\Ab\in \RR^{p\times q}}{\mathrm{minimize}} \;\bigg\{{\frac{1}{n}\ell_\tau  \left(\Yb-\Xb \Ab \right)}+\lambda \left( \|\Ab\|_*+\gamma\|\Ab\|_{1,1}\right)\bigg\},
\end{equation}
where $\lambda$ and $\gamma$ are non-negative tuning parameters,  $\|\cdot\|_*$ is the nuclear norm that encourages the solution to be low rank, and $\|\cdot\|_{1,1}$ is the entry-wise $\ell_1$-norm that encourages the solution to be sparse. 
The nuclear norm and the $\ell_{1,1}$ norm constraints are the tightest convex relaxations of the rank and cardinality constraints, respectively \citep{MaryamNuclear, jojic2011convex}. 
In Section \ref{theory}, we will show that the estimator obtained from solving the convex relaxation in \eqref{eq:opt:conv} has a favorable statistical convergence rate under a bounded moment condition on the random noise.

\subsection{Algorithm}
We now develop an alternating direction method of multipliers (ADMM) algorithm for solving \eqref{eq:opt:conv}, which allows us to decouple some of the terms that are difficult to optimize jointly \citep{EcksteinADMM92,BoydADMM}.
More specifically, (\ref{eq:opt:conv}) is equivalent to 
\begin{equation}
\label{eq:opt:conv2}
\begin{split}
&\underset{\Ab,\Zb,\Wb \in \RR^{p\times q}, \Db \in \RR^{n\times q}}{\mathrm{minimize}} ~~~\bigg\{{\frac{1}{n}\ell_\tau \left(\Yb- \Db \right)}+\lambda \left( \|\Wb\|_*+\gamma\|\Zb\|_{1,1}\right)\bigg\},\\
& \textnormal{subject to} ~~ \begin{pmatrix} \Db \\ \Zb \\ \Wb  \end{pmatrix}=\begin{pmatrix}  \Xb  \\ \Ib \\ \Ib  \end{pmatrix}\Ab.
\end{split}
\end{equation}
For notational convenience, let $\Bb = (\Bb_D,\Bb_Z,\Bb_W)^\T$, $\tilde{\Xb} = ( \Xb,\Ib,\Ib)^\T$, and $\bOmega = (\Db,\Zb,\Wb)^\T$.
The scaled augmented Lagrangian of (\ref{eq:opt:conv2}) takes the form
\[
\cL_\rho(\Ab,\Db,\Zb,\Wb,\Bb) =  \frac{1}{n} \ell_{\tau} (\Yb-\Db)+\lambda \left( \|\Wb\|_*+\gamma\|\Zb\|_{1,1}\right)+\frac{\rho}{2} \|\bOmega-\tilde{\Xb}\Ab+\Bb\|_{\rF}^2,
\]
where $\Ab, \Db, \Zb, \Wb$ are the primal variables, and $\Bb$ is the dual variable.
Algorithm~\ref{Alg:general} summarizes the ADMM algorithm for solving (\ref{eq:opt:conv2}). A detailed derivation is deferred to Appendix~\ref{appendix:alg}.
Note that the term $(\tilde{\Xb}^\T \tilde{\Xb})^{-1}$ can be calculated before Step 2 in Algorithm \ref{Alg:general}.
Therefore, the computational bottleneck in each iteration of Algorithm~\ref{Alg:general} is the singular value decomposition of a $p\times q$ matrix with computational complexity $\cO(p^2q + q^3)$.

\begin{algorithm}[!t]
\small
\caption{An ADMM Algorithm for Solving \eqref{eq:opt:conv2}.}
\label{Alg:general}
\begin{enumerate}
\item  \textbf{Initialize} the parameters: 
\begin{enumerate}
\item primal variables $\Ab, \Db, \Zb$, and $\Wb$ to the zero matrix.
\item dual variables $\Bb_D,\Bb_Z$, and $\Bb_W$ to the zero matrix.
\item  constants $\rho>0$ and $\epsilon>0$. 
\end{enumerate}
\item  \textbf{Iterate} until the stopping criterion ${\|\Ab^{t}- \Ab^{t-1} \|_\rF^2}/{\| \Ab^{t-1}\|_\rF^2} \le \epsilon$ is met, where $\Ab^t$ is the value of $\Ab$ obtained at the $t$th iteration:

\begin{enumerate}
\item Update ${\Ab}, \Zb, \Wb, \Db$:
\begin{enumerate}
\item $\Ab = (\tilde{\Xb}^{\T} \tilde{\Xb})^{-1} \tilde{\Xb}^{\T}(\bOmega +\Bb)$.
\item $\Zb = S(\Ab-\Bb_Z,\lambda \gamma /\rho)$.  Here $S$ denote the soft-thresholding operator, applied element-wise to a matrix:  $S(A_{ij},b) = \text{sign}(A_{ij}) \max( |A_{ij}|-b, 0)$.
\item  $\Wb = \sum_{j}   \max \left( \omega_j - \lambda/\rho  ,0   \right) \ab_j \mathbf{b}_j^\T $, where $\sum_{j}\omega_j \ab_j  \mathbf{b}_j^\T$ is the singular value decomposition of $\Ab-\Bb_W$.  
\item $\Cb = \Xb\Ab - \Bb_D$.  Set
\[
D_{ij} =\begin{cases} (Y_{ij}+n\rho C_{ij})/({1+n\rho}), & \mathrm{if} ~ \left| {n\rho (Y_{ij}-C_{ij})}/({1+n\rho})\right| \le \tau,\\
Y_{ij} - S(Y_{ij} -C_{ij},\tau/(n\rho)), & \mathrm{otherwise}.
\end{cases}
\]

\end{enumerate}

\item Update  $\Bb_D, \Bb_Z, \Bb_W$:
\begin{enumerate}
\item $\Bb_D =\Bb_D + \Db -  \Xb \Ab$;   \;  \;\;  ii. $\Bb_Z =\Bb_Z + \Zb -  \Ab$; \;  \;\;  iii. $\Bb_W =\Bb_W + \Wb -  \Ab$.
\end{enumerate}

\end{enumerate}
\end{enumerate}
\end{algorithm}

\section{Statistical Theory}
\label{theory}
We study the theoretical properties of $\hat{\Ab}$ obtained from  solving~(\ref{eq:opt:conv}).
Let $\VV_{p,q} = \{\Ub\in \RR^{p\times q} : \Ub^\T \Ub = \Ib_q\} $ be the Stiefel manifold of $p\times q$ orthonormal matrices. 
Throughout the theoretical analysis, we assume that $\Ab^*$ can be decomposed as
\begin{equation}
\label{Eq:star}
\Ab^*=\Ub^*\bLambda^*(\Vb^*)^\T=\sum_{k=1}^{r} \lambda_k^* \ub_k^* (\vb_k^*)^\T, 
\end{equation}
where $\Ub^*\in \VV_{p,r}$,  $\Vb^*\in \VV_{q,r}$,   $\max_{k} \|\ub_k^*\|_0 \le s_u$, and $\max_k \|\vb_k^*\|_{0} \le s_v$ with $s_u, s_v \ll n$, $r \ll n$, and $r s_u s_v \ll n$.
Consequently, $\Ab^*$ is sparse and low rank. 
Let  $\cS = \mathrm{supp}(\Ab^*)$ be the support set of $\Ab^*$ with cardinality $|\cS|=s$, i.e., $\cS$ contains indices for the non-zero elements in $\Ab^*$. 
Note that $s \le r s_u s_v$.

For simplicity, we consider the case of fixed design matrix $\Xb$ and assume that the covariates are 
standardized such that $\max_{i,j}|X_{ij}|= 1$.  
To characterize the heavy-tailed random noise, we impose a bounded moment condition on the random noise. 
\begin{condition}[\sf Bounded Moment Condition]
\label{condition:error}
For $\delta > 0$, each entry of the random error matrix $\Eb$ in \eqref{eq:main} has bounded $(1+\delta)$th moment 
\$
v_{\delta}\equiv \max_{i,j} \EE\big(|E_{ij}|^{1+\delta}\big)<\infty.
\$ 
\end{condition}
\noindent  Condition~\ref{condition:error}  is a relaxation of the commonly used sub-Gaussian assumption to  accommodate heavy-tailed random noise. 
For instance, the $t$-distribution with degrees of freedom larger than one can be accommodated by the bounded moment condition. 
This condition has also been used in the context of high-dimensional Huber linear regression  \citep{sun2016adaptive}.

Let $\Hb_\tau (\Ab)$ be the Hessian matrix of the Huber loss function $\ell_\tau  \left(\Yb-\Xb \Ab \right)/n$ in \eqref{eq:opt:conv2}.
In addition to the random noise, the Hessian matrix is a function of the parameter $\Ab$, and  $\Hb_\tau (\Ab)$ may equal zero for some $\Ab$, because the Huber loss is linear at the tails. 
To avoid singularity of $\Hb_\tau(\Ab)$, we will study the Hessian matrix in a local neighborhood of $\Ab^*$.   
To this end, we define and impose conditions on the localized restricted eigenvalues of $\Hb_\tau(\Ab)$.
\begin{definition}[\sf Localized Restricted Eigenvalues]
\label{def:restricted}
  The minimum and maximum localized restricted eigenvalues  for $\Hb_{\tau}(\Ab)$ are defined as
\$
\kappa_{-} (\Hb_{\tau} (\Ab),\xi,\eta) 
&=  \underset{\Ub,\Ab}{\inf} \left\{
\frac{\vecc(\Ub)^\T \Hb_{\tau}(\Ab)\vecc(\Ub)}{\|\Ub\|_\rF^2}:  (\Ab,\Ub) \in \cC  (m,\xi,\eta) 
\right\},\\
\kappa_{+} (\Hb_{\tau} (\Ab),\xi,\eta) 
&=  \underset{\Ub,\Ab}{\sup} \left\{
\frac{\vecc(\Ub)^\T \Hb_{\tau}(\Ab)\vecc(\Ub)}{\|\Ub\|_\rF^2}:  (\Ab,\Ub) \in \cC  (m,\xi,\eta) 
\right\},
\$
where 
\[
\mathcal{C}(m,\xi,\eta) \!=\! \{  (\Ab,\Ub)\in \RR^{p\times q}\times \RR^{p\times q}:  \Ub \ne \mathbf{0}, \cS \subseteq J, |J| \le m, \|\Ub_{\cS^c}\|_{1,1} \le \xi \| \Ub_{\cS}  \|_{1,1}, \|\Ab-\Ab^*  \|_{1,1} \le \eta  \}
\]
 is a \emph{local $\ell_{1,1}$-cone}.
\end{definition}

\begin{condition}
\label{condition:lre}
There exist constants $0 < \kappa_{\mathrm{lower}}\le \kappa_{\mathrm{upper}} <\infty$ such that the localized restricted eigenvalues of $\Hb_{\tau}$ are lower-and upper-bounded by
\$
 \kappa_{\mathrm{lower}}/2\le\kappa_{-} (\Hb_{\tau} (\Ab),\xi,\eta)  \le \kappa_{+} (\Hb_{\tau} (\Ab),\xi,\eta) \le \kappa_{\mathrm{upper}}.
\$
\end{condition}

A similar type of localized condition was proposed  in \cite{fan2015tac} for general loss functions and in \cite{sun2016adaptive} for the analysis of robust linear regression in high dimensions. 
In what follows, we justify Condition~\ref{condition:lre} by showing that it is implied by the restricted eigenvalue condition on the empirical Gram matrix $\mathbf{S}= \Xb^\T \Xb /n$.
To this end, we define the restricted eigenvalues of a matrix and then place a condition on the restricted eigenvalues of $\mathbf{S}$.

\begin{definition}[\sf Restricted Eigenvalues of a Matrix]
\label{def:restricted eigenvalue}
 Given $\xi>1$, the minimum and maximum restricted  eigenvalues of $\mathbf{S}$ are defined as 
\[
\rho_{-} ({\mathbf{S}},\xi,m) =  \underset{\Ub}{\inf} \left\{
\frac{\mathrm{tr}(\Ub^\T {\mathbf{S}}\Ub)}{\|\Ub\|_{1,2}^2} :\Ub \in \RR^{p\times q}, \Ub \ne \mathbf{0}, \cS \subseteq J, |J|\le m, \|\Ub_{J^c}\|_{1,1} \le \xi \| \Ub_{J}  \|_{1,1}
\right\},
\]
\[
\rho_{+} ({\mathbf{S}},\xi,m) =  \underset{\Ub}{\sup} \left\{
\frac{\mathrm{tr}(\Ub^{\T}{\mathbf{S}}\Ub)}{\|\Ub\|_{1,2}^2} :\Ub \in \RR^{p\times q}, \Ub \ne \mathbf{0}, \cS \subseteq J, |J|\le m, \|\Ub_{J^c}\|_{1,1} \le \xi \| \Ub_{J}  \|_{1,1}
\right\},
\]
respectively.
\end{definition}

\begin{condition}
\label{condition2}
There exist constants $0<\kappa_{\mathrm{lower}} \le \kappa_{\mathrm{upper}}<\infty$ such that 
the restricted eigenvalues of $\mathbf{S}$ are lower- and upper-bounded by 
\[
 \kappa_{\mathrm{lower}} \le \rho_{-} ({\mathbf{S}},\xi,m)\le\rho_{+} ({\mathbf{S}},\xi,m) \le \kappa_{\mathrm{upper}}.
 \]
\end{condition}

\noindent Condition~\ref{condition2} is a variant of the restricted eigenvalue condition that is commonly used in high-dimensional non-asymptotic analysis. It can be shown that Condition~\ref{condition2} holds with high probability if each row of $\Xb$ is a sub-Gaussian random vector.

Under Condition~\ref{condition2}, we now show that the localized restricted eigenvalues for the Hessian matrix are bounded with high probability under conditions on the robustification parameter $\tau$ and the sample size $n$. That is, we prove that the localized restricted eigenvalues condition in Condition~\ref{condition:lre} holds with high probability under Condition~\ref{condition2}.
 The result is summarized in the following lemma. 
\begin{lemma}
\label{lemma:localized restricted}
Consider $\Ab \in \cC(m,\xi,\eta)$ where $\cC(m,\xi,\eta)$ is the local $\ell_{1,1}$-cone as defined in Definition~\ref{def:restricted}.  Let $\tau \ge \min(8\eta, C \cdot (m \nu_{\delta})^{1/(1+\delta)})$ and let $n > C'\cdot m^2 \log(pq)$ for sufficiently large constants $C,C'>0$.
Under Conditions~\ref{condition:error} and \ref{condition2}, there exists constants $\kappa_{\mathrm{lower}}$ and $\kappa_{\mathrm{upper}}$ such that the localized restricted eigenvalues of $\Hb_{\tau} (\Ab)$ satisfy 
\[
0 < \kappa_{\mathrm{lower}}/2\le\kappa_{-} (\Hb_{\tau} (\Ab),\xi,\eta)  \le \kappa_{+} (\Hb_{\tau} (\Ab),\xi,\eta) \le \kappa_{\mathrm{upper}} < \infty
\]
 with probability at least $1-(pq)^{-1}$.
\end{lemma}
\noindent Lemma~\ref{lemma:localized restricted} shows that Condition~\ref{condition:lre} holds with high probability, as long as Condition~\ref{condition2} on the empirical Gram matrix $\mathbf{S}$ holds.
Note that the constants $\kappa_{\mathrm{lower}}$ and $\kappa_{\mathrm{upper}}$ also appear in Condition~\ref{condition2}.

We now present our main results on the estimation error of $\hat{\Ab}$ under the Frobenius norm and nuclear norm in the following theorem. For simplicity, we will present our main results conditioned on the event that  Conditions~\ref{condition:error}--\ref{condition:lre} hold.
\begin{theorem}
\label{thm:up}
Let $\hat{\Ab}$ be a solution to \eqref{eq:opt:conv} with truncation and tuning parameters
\[
\tau \gtrsim \left(\frac{n v_{\delta}}{\log (pq)}\right)^{1/\min \{(1+\delta),2\}}, \qquad \lambda \gtrsim   v_{\delta}^{1/\min(1+\delta,2)}    \left( \frac{
\log (pq)}{n}\right)^{\min\{{\delta}/(1+\delta),1/2\}}
\]
and $\gamma>2.5$. Suppose that Conditions~\ref{condition:error}--\ref{condition:lre}   hold with $\xi = (2\gamma+5)/(2\gamma-5)$,  $\kappa_{\mathrm{lower}}>0$ and 
$
\eta\gtrsim\kappa_{\mathrm{lower}}^{-1}  \lambda s.
$ Assume that $n > C s^2 \log (pq)$ for some sufficiently large universal constant $C>0$. Then, with probability at least $1-(pq)^{-1}$, we have
\$
\big\|\widehat \Ab-\Ab^*\big\|_\rF&\lesssim  \kappa^{-1}_{\mathrm{lower}} v_{\delta}^{1/\min\{1+\delta,2\}}  \sqrt{rs_u s_v}~ \bigg\{\frac{\log (p q)}{n}\bigg\}^{\min\{\delta/(1+\delta), 1/2\}},\\
\big\|\widehat \Ab-\Ab^*\big\|_*&\lesssim   \kappa^{-1}_{\mathrm{lower}}v_{\delta}^{1/\min\{1+\delta,2\}}  rs_u s_v~ \bigg\{\frac{\log (p q)}{n}\bigg\}^{\min\{\delta/(1+\delta), 1/2\}}.
\$
\end{theorem}

Theorem~\ref{thm:up} establishes the non-asymptotic  convergence rates  of our proposed estimator under both Frobenius and nuclear norms in the high-dimensional setting. To the best of our knowledge, we are the first to establish such results on the estimation error for robust sparse reduced rank regression.  
By contrast, most of the existing work on reduced rank regression focuses on  rank selection consistency and prediction consistency \citep{bunea2011optimal,bunea2012joint}.  Moreover, the prediction consistency results in \citet{she2017robust} are established under the assumption that the rank of the design matrix $\Xb$ is smaller than the number of observations $n$.
When the random noise has second or higher moments, i.e., $\delta\ge 1$, our proposed estimator achieves a parametric rate of convergence as if sub-Gaussian random noise were assumed.
It achieves a slower rate of convergence only when the random noise is extremely heavy-tailed, i.e., $0 < \delta <1$.

Intuitively, one might expect the optimal rate of convergence under the  Frobenius norm to have the form 
\$
\big\|\widehat \Ab-\Ab^*\big\|_\rF&\lesssim \sqrt{r(s_u\!+\!s_v)}~ \bigg\{\frac{\log (p q)}{n}\bigg\}^{\min\{\delta/(1+\delta), 1/2\}},
\$ 
since there are a total of roughly $r(s_u+s_v)$ nonzero parameters to be estimated in $\Ab^*$ as defined in \eqref{Eq:star}. Using the convex relaxation \eqref{eq:opt:conv}, we gain  computational tractability while losing a scaling factor of $\sqrt{s_us_v/(s_u\!+\!s_v)}$. 

By defining the {\it effective dimension} as $d_{\reff}=rs_u s_v$ and the {\it effective sample size} as $n_{\reff}=\big\{n/\log \big(p q)\big\}^{\min\{2\delta/(1+\delta), 1\}}$, the upper bounds in Theorem \ref{thm:up} can be rewritten as 
\$
\big\|\widehat \Ab-\Ab^*\big\|_\rF\lesssim  \sqrt{\frac{d_\reff}{n_\reff}},\quad 
\big\|\widehat \Ab-\Ab^*\big\|_*\lesssim \frac{ d_\reff}{ \sqrt{n_\reff}}. 
\$
The {\it effective dimension} depends only on the sparsity and rank, while the {\it effective sample size} depends only on the sample size divided by the log of the number of free parameters, as if there were no structural constraints.
Our results exhibit an interesting phenomenon: the rate of convergence is affected by the heavy-tailedness only through the effective sample size; the effective dimension stays the same regardless of $\delta$.
This parallels results for Huber linear regression in \citet{sun2016adaptive}.

\section{Numerical Studies}
\label{sec:sim}
We perform extensive numerical studies to evaluate the performance of our proposal for robust sparse reduced rank regression.
Five approaches are compared in our numerical studies: our proposal with Huber loss, \texttt{hubersrrr}; our proposal with squared error loss (with $\tau\rightarrow \infty$), \texttt{srrr}; robust reduced rank regression with an additional mean parameter that models the outliers \citep{she2017robust}, \texttt{r4}; penalized reduced rank regression via an adaptive nuclear norm \citep{chen2013reduced}, \texttt{rrr}; and the penalized reduced rank regression via a ridge penalty  \citep{mukherjee2011reduced}, \texttt{rrridge}. 
The proposals \texttt{rrridge}, \texttt{rrr}, and \texttt{r4} do not assume sparsity on the regression coefficients.  Moreover, \texttt{r4} can only be implemented in the low-dimensional setting in which $n \ge p$, or under the assumption that the design matrix $\Xb$ is low rank.    
Among the five proposals, only \texttt{hubersrrr} and \texttt{r4} are robust against outliers. 

For all of our numerical studies, we generate each row of $\Xb$ from a multivariate normal distribution with mean zero and covariance matrix $\bSigma$, where $\bSigma_{ij} = 0.5^{|i-j|}$ for $1\le i,j\le p$. Then, all elements of $\Xb$ are divided by the maximum absolute value of $\Xb$ such that $\max_{i,j} |X_{ij}| = 1$.  
The response matrix $\Yb$ is then generated according to $\Yb = \Xb \Ab^* +\Eb$.  
We consider two different types of outliers: (i) heavy-tailed random noise $\Eb$, and (ii) contamination of some percentage of the elements of $\Yb$.
We simulate data with sparse and non-sparse low rank matrix $\Ab^*$.
The details for the different scenarios will be specified in Section~\ref{sim:matrixlow}.

Our proposal \texttt{hubersrrr} involves three tuning parameters.  We select the tuning parameters using five-fold cross-validation: we vary $\lambda$ across a fine grid of values, consider four values of $\gamma =\{2.5,3,3.5,4\}$ as suggested by Theorem~\ref{thm:up}, and considered a range of the robustification parameter $\tau = c \{ n/\log(pq)\}^{1/2}$, where $c=\{0.4,0.45,\ldots,1.45,1.5\}$.  The tuning parameters for \texttt{srrr} are selected in a similar fashion with $\tau \rightarrow \infty$. 
For scenarios with non-sparse regression coefficients, we simply set $\gamma=0$ for \texttt{hubersrrr} and \texttt{srrr} for fair comparison against other approaches that do not assume sparsity. 
For \texttt{r3}, we select the tuning parameter using five different information criteria implemented in the \texttt{R} package \texttt{rrpack} \citep{chen2013reduced}, and report the best result.  
 For \texttt{rrridge}, we specify the correct rank for $\Ab^*$ and simply consider a fine grid of tuning parameters for the ridge penalty and report the best result.  
The two tuning parameters for \texttt{r4} control the sparsity of the mean shift parameter for modeling outliers, and the rank of $\Ab^*$.  We implement \texttt{r4} by specifying the correct rank of $\Ab^*$, and choose the sparsity tuning parameter according to five-fold cross-validation.
In other words, we give a major advantage to \texttt{rrridge} and \texttt{r4}, in that we provide the rank of $\Ab^*$ as an input.

To evaluate the performance across different methods, we calculate the difference between the estimated regression coefficients $\hat{\Ab}$ and the true coefficients $\Ab^*$ under the Frobenius norm.  
In addition, for scenarios with in which $\Ab^*$ is sparse, we calculate the true and false positive rates (TPR and FPR), defined as the proportion of correctly estimated nonzeros in the true parameter, and the proportion of zeros that are incorrectly estimated to be nonzero in the true parameter, respectively.

Since some existing approaches are not applicable in the high-dimensional setting, we perform numerical studies under the low-dimensional setting in which $n\ge p$ in Section~\ref{sim:matrixlow}.  We then illustrate the performance of our proposed methods, \texttt{hubersrrr} and \texttt{srrr}, in the high-dimensional setting in Section~\ref{sim:matrix}.

\subsection{Low-Dimensional Setting with $n\ge p$}
\label{sim:matrixlow}
In this section, we perform numerical studies with $n=200$, $p=50$, and $q=10$. We first consider two cases in which $\Ab^*$ has low rank but is not sparse:
\begin{enumerate}
\item Rank one matrix: $\Ab^* =  \mathbf{u}_1\mathbf{v}_{1}^\T$, where each element of $\mathbf{u}_1 \in \mathbb{R}^p$ and $\mathbf{v}_1\in \mathbb{R}^q$ is generated from a uniform distribution on the interval $[-1,0.5]\cup [0.5,1]$.

\item Rank two matrix: $\Ab^* =   \mathbf{u}_1\mathbf{v}_{1}^\T+  \mathbf{u}_2\mathbf{v}_{2}^\T$, where each element of  $\ub_1,\ub_2 \in \mathbb{R}^p$ and $\vb_1,\vb_2\in \mathbb{R}^q$  is generated from a uniform distribution on the interval $[-1,0.5]\cup [0.5,1]$.
\end{enumerate}
We then generate random noise $\Eb \in \mathbb{R}^{n\times q}$ from three different distributions: (i) the normal distribution $N(0,4)$, (ii) the $t$-distribution with degrees of freedom 1.5, and (iii) the log-normal distribution $\log N(0, 1.2^2)$.  
Moreover, we consider a contamination scenario in which we generate each element of $\Eb$ from the $N(0,4)$ distribution, and then randomly contaminate $5\%$ and $10\%$ of the elements in $\Yb$ by replacing them with random values generated from a uniform distribution on the interval $[10,20]$.
The estimation error for each method under the Frobenius norm, averaged over 100 data sets, is reported in Table~\ref{tablelowd1}.

From Table~\ref{tablelowd1}, we see that \texttt{rrr} and \texttt{rrridge} outperform all  other methods when $\Ab^*$ is rank one under Gaussian noise. This is not surprising, since \texttt{rrr} and \texttt{rrridge} are tailored for reduced rank regression without outliers.     
We see that \texttt{hubersrrr} has similar performance to \texttt{srrr}, suggesting that there is no loss of efficiency for \texttt{hubersrrr} even when there are no outliers.  
When the random noise is generated from the $t$-distribution, \texttt{r4} has the best performance, followed by \texttt{hubersrrr}.  The estimation errors for methods that do not model the outliers are substantially higher.   For log-normal random noise, \texttt{hubersrrr} outperforms \texttt{r4}.
Under the data contamination model, \texttt{r4} and \texttt{hubersrrr} perform similarly, and both outperform all of the other methods.
These results corroborate the observation in \citet{she2017robust} that the estimation of low rank matrices is extremely sensitive to outliers.  
As we increase the contamination percentage of the observed outcomes, we see that the performance of the non-robust methods deteriorates. 
 Similar results are observed for the case when $\Ab^*$ has rank two.

\begin{table}[!t]
\footnotesize
\begin{center}
\caption{The mean (and standard error) of the difference between the estimated regression coefficients and the true regression coefficients under the Frobenius norm, averaged over 100 data sets, in the setting where $\Ab^*$ is not sparse, with $n=200$, $p=50$, and $q=10$. Three distributions of random noise are considered: normal, $t$, and log-normal.  We also considered contaminating 5\% or 10\% of the elements of $\Yb$.    }
\begin{tabular}{c| l|    c cc| cccc}
  \hline
Rank of $\Ab^*$ &  & \multicolumn{3}{c}{Random Noise}& \multicolumn{3}{c}{Data Contamination} \\\hline
&Methods&Normal&$t$ &Log-normal&0\%&5\%&10\% \\ 
\hline
&\texttt{rrr}   & 5.80 (0.07)& 17.71 (2.84)& 10.71 (0.19)&  5.80 (0.07)  & 10.35 (0.13) & 12.33 (0.11)  \\ 
 &\texttt{rrridge} & 5.42 (0.06) & 13.79 (0.52) & 9.22 (0.17)&  5.42 (0.06) & 9.07 (0.10)& 10.93 (0.11)\\
1&\texttt{srrr} & 7.19 (0.08) & 26.75 (5.32)&  10.41 (0.13)&7.19 (0.08)&  10.49 (0.09)& 11.76 (0.10)\\
&\texttt{r4} & 7.32 (0.10) & 4.65 (0.07)&  8.88 (0.16)& 7.32 (0.10) & 7.93 (0.11)& 8.54 (0.12)\\
&\texttt{hubersrrr} & 7.21 (0.08) & 6.96 (0.13)&  6.70 (0.08)&7.21 (0.08)& 7.92 (0.09)& 8.40 (0.09)\\
\hline
&\texttt{rrr}   & 6.09 (0.09)& 31.20 (5.64)& 12.08 (0.32)& 6.09 (0.09)  & 12.29 (0.19) & 16.81 (0.25) \\ 
&\texttt{rrridge} & 9.16 (0.09) & 22.75 (1.16) & 15.16 (0.20)&  9.16 (0.09) & 15.24 (0.12) & 18.22 (0.13)\\
2 &\texttt{srrr} & 8.69 (0.11) & 41.76 (11.42)&  14.20 (0.24)& 8.69 (0.11)& 14.94 (0.16)  & 18.26 (0.18)\\
&\texttt{r4} & 11.63 (0.13) & 8.51 (0.41)&  12.56 (0.17)&11.63 (0.13) & 12.62 (0.14) & 13.69 (0.15) \\
&\texttt{hubersrrr} & 8.70 (0.11) & 8.25 (0.24)&  7.82 (0.11)& 8.70 (0.11) & 9.81 (0.13) & 10.99 (0.15) \\
\hline
\end{tabular}
\label{tablelowd1}
\end{center}
\end{table}

Next, we consider two cases in which $\Ab^*$ is both sparse and low rank:  
\begin{enumerate}
\item Sparse rank one matrix: $\Ab^* = \mathbf{u}_1\mathbf{v}_{1}^\T$ with $\mathbf{u}_1 = (\mathbf{1}_4^\T,\mathbf{0}_{p-4}^\T)^\T$ and $\mathbf{v}_1 = (\mathbf{1}_{4}^\T,\mathbf{0}_{q-4}^\T)^\T$; 
\item Sparse rank two matrix: $\Ab^* = \mathbf{u}_1\mathbf{v}_{1}^\T+\mathbf{u}_2\mathbf{v}_{2}^\T$ with $\mathbf{u}_1 = (\mathbf{1}_4^\T,\mathbf{0}_{p-4}^\T)^\T$, $\mathbf{v}_1 = (\mathbf{1}_{4}^\T,\mathbf{0}_{q-4}^\T)^\T$,  $\mathbf{u}_2 = (\mathbf{0}_2^\T,\mathbf{1}_4^\T,\mathbf{0}_{p-6}^\T)^\T$, and $\mathbf{v}_2 = (\mathbf{0}_2^\T,\mathbf{1}_4^\T,\mathbf{0}_{q-6}^\T)^\T$.
\end{enumerate}
The heavy-tailed random noise and data contamination scenarios are as described earlier.
The results, averaged over 100 data sets, are reported in Table~\ref{tablelowd2}.

When $\Ab^*$ is sparse, \texttt{hubersrrr} and \texttt{srrr} outperform all of the methods that do not assume sparsity.  In particular, we see that \texttt{r4} has the worst performance when the random noise is normal or log-normal, or when the data are contaminated.    The method \texttt{rrr} has an MSE of 5.00 when the data are contaminated, due to the fact that the information criteria always select models with the regression coefficients estimated to be zero.   
In short, our proposal  \texttt{hubersrrr} has the best performance across all scenarios and is robust against different types of outliers.

\begin{table}[!t]
\footnotesize
\begin{center}
\caption{Results for the case where $\Ab^*$ is sparse, with $n=200$, $p=50$, and $q=10$. Other details are as in Table~\ref{tablelowd1}. }
\begin{tabular}{c| l|    c cc| cccc}
  \hline
rank of $\Ab^*$&  & \multicolumn{3}{c}{Random Noise}& \multicolumn{3}{c}{Data Contamination} \\\hline
&Methods &Normal&$t$-dist &Log-normal&0\%&5\%&10\% \\ 
\hline
&\texttt{rrr}   & 4.65 (0.04)& 6.95 (0.88)& 4.98 (0.01)&  4.65 (0.04)  & 5.00 (0.01) & 5.00 (0.01)  \\ 
 &\texttt{rrridge} & 2.73 (0.03) & 7.78 (0.55) & 4.17 (0.08)&  2.73 (0.03) & 4.02 (0.04)& 4.64 (0.05)\\
1&\texttt{srrr} & 2.57 (0.04) & 5.02 (0.08)&  4.48 (0.06)& 2.54 (0.04)&  4.54 (0.04)& 4.94 (0.04)\\
&\texttt{r4} & 7.29 (0.10) & 4.79 (0.09)&  10.44 (0.16)&  7.29 (0.10) & 7.98 (0.12)& 8.94 (0.12)\\
&\texttt{hubersrrr} & 2.57 (0.04) & 2.82 (0.13)&  2.37 (0.05)& 2.54 (0.04)& 2.93 (0.05)& 3.28 (0.06)\\
\hline
&\texttt{rrr}   & 5.25 (0.04)& 10.05 (0.82)& 8.18 (0.03)& 5.25 (0.04)  & 8.22 (0.01) & 8.24 (0.01) \\ 
 &\texttt{rrridge} & 4.36 (0.03) & 9.35 (0.54) & 6.00 (0.05)& 4.36 (0.03) & 6.11 (0.04) & 6.82 (0.04)\\
2&\texttt{srrr} & 3.26 (0.04) & 7.81 (0.11)&  5.78 (0.11)&  3.26 (0.04)& 5.83 (0.06)  & 6.96 (0.07)\\
&\texttt{r4} & 11.55 (0.12) & 7.75 (0.12)&  12.91 (0.15) & 11.55 (0.12) & 12.50 (0.13) & 13.65 (0.15) \\
&\texttt{hubersrrr} & 3.27 (0.04) & 3.44 (0.11)&  3.07 (0.05)& 3.27 (0.04) & 3.70 (0.04) & 4.09 (0.06) \\
\hline
\end{tabular}
\label{tablelowd2}
\end{center}
\end{table}

\subsection{High-Dimensional Setting with $p> n$}
\label{sim:matrix}
In this section, we assess the performance of our proposed method in the high-dimensional setting, when the matrix $\Ab^*$ is sparse.  
To this end, we perform numerical studies with $q=10$, $p=200$, and $n=150$. 
Note that \texttt{r4} is not applicable when $p>n$.   
Moreover,  \texttt{rrr} and \texttt{rrridge} do not assume sparsity and therefore their results are omitted.
We consider low rank and sparse matrices $\Ab^*$ described in Section~\ref{sim:matrixlow}.     
Similarly, two types of outliers are considered: heavy-tailed random noise, and data contamination.
The TPR, FPR, and estimation error under Frobenius norm for both types of scenarios, averaged over 100 data sets, are summarized in Tables~\ref{table3}--\ref{table4}, respectively. 

We see that for Gaussian random noise,  \texttt{hubersrrr} is comparable to \texttt{srrr}, indicating that there is little loss of efficiency when there are no outliers.  
However, in scenarios in which the random noise is heavy-tailed, \texttt{hubersrrr} has high TPR, low FPR, and low Frobenius norm compared to \texttt{srrr}.
In fact, we see that when the random noise is heavy-tailed, the TPR and FPR of \texttt{srrr} are approximately zero.
We see similar performance for the case when the data are contaminated in Table~\ref{table4}.  
These results suggest that \texttt{hubersrrr} should be preferred in all scenarios since it allows accurate estimation of $\Ab^*$ when the random noise are heavy-tailed, or under  data contamination.  Moreover, there is little loss of efficiency compared to \texttt{srrr} when there are no outliers.

\begin{table}[!t]
\footnotesize
\begin{center}
\caption{Results for the case when $\Ab^*$ is sparse and low rank in the high-dimensional setting with $n=150$, $p=200$, and $q=10$.  Three distributions of random noise are considered: normal, $t$, and log-normal. We report the mean (and standard error) of the true and false positive rates, and the difference between $\hat{\Ab}$ and $\Ab^*$ under Frobenius norm, averaged over 100 data sets.  }

\begin{tabular}{c| l|    c cc| cccc}
  \hline
Rank of $\Ab^*$& Noise & \multicolumn{3}{c}{\texttt{srrr}}& \multicolumn{3}{c}{\texttt{hubersrrr}} \\\hline
&&TPR&FPR&Frobenius&TPR&FPR&Frobenius \\ 
\hline
&Normal   & 0.95 (0.01)& 0.12 (0.01)& 3.74 (0.05)& 0.95 (0.01)  & 0.13 (0.01) & 3.75 (0.05)  \\ 
1 &$t$-dist & 0.01 (0.01) & 0.01 (0.01) & 6.23 (1.23)& 0.96 (0.02) & 0.14 (0.01)& 4.28 (0.47)\\
&Log-normal & 0.08 (0.02) & 0.01 (0.01)&  5.00 (0.02)&0.98 (0.01)& 0.15 (0.01)& 3.53 (0.06)\\
\hline
&Normal  & 0.96 (0.01) & 0.15 (0.01)& 4.65 (0.05)&0.96 (0.01)& 0.16 (0.01)& 4.65 (0.05)\\
2&$t$-dist &0.06 (0.02)& 0.01 (0.01)& 9.38 (1.21)& 0.97 (0.01)& 0.17 (0.01)& 5.11 (0.48)\\
&Log-normal & 0.39 (0.03)& 0.03 (0.01)& 7.50 (0.09) &0.98 (0.01)&0.18 (0.01)& 4.41 (0.07)\\
\hline
\end{tabular}
\label{table3}
\end{center}
\end{table}

\begin{table}[!t]
\footnotesize
\begin{center}
\caption{ Results for the case when $\Ab^*$ is sparse and low rank, and $n=150$, $p=200$, and $q=10$, with 5\% and 10\% of the data being contaminated.  Other details are as in Table~\ref{table3}.  }
\begin{tabular}{c| c|    c cc| cccc}
  \hline
Rank of $\Ab^*$& Contamination \%& \multicolumn{3}{c}{\texttt{srrr}}& \multicolumn{3}{c}{\texttt{hubersrrr}} \\\hline
&&TPR&FPR&Frobenius&TPR&FPR&Frobenius \\ 
\hline
&0\%   & 0.95 (0.01)& 0.12 (0.01)& 3.74 (0.05)& 0.95 (0.01)  & 0.13 (0.01) & 3.75 (0.05)  \\ 
1 &5\% & 0.14 (0.02) & 0.02 (0.01) & 5.08 (0.03)& 0.82 (0.03) & 0.12 (0.01)& 4.24 (0.06)\\
&10\% & 0.04 (0.01) & 0.01 (0.01)&  5.13 (0.04)&0.74 (0.03)& 0.11 (0.01)& 4.52 (0.06)\\
\hline
&0\%  & 0.96 (0.01) & 0.15 (0.01)& 4.65 (0.05)&0.96 (0.01)& 0.16 (0.01)& 4.65 (0.05)\\
2&5\% &0.49 (0.02)& 0.06 (0.01)& 7.43 (0.07)& 0.94 (0.01)& 0.15 (0.01)& 5.22 (0.07)\\
&10\% & 0.21 (0.02)& 0.03 (0.01)& 8.13 (0.04) &0.90 (0.01)&0.15 (0.01)& 5.63 (0.08)\\
\hline
\end{tabular}
\label{table4}
\end{center}
\end{table}

\section{Data Application}
\label{realdata:rsrrr}
We apply the proposed robust sparse reduced rank regression to the \emph{Arabidopsis thaliana} data set, which consists of gene expression measurements for $n=118$ samples \citep{Rodriguez02,WilleEtAl2004,Maetal2007,tanetal2015,she2017robust}. 
It is known that isoprenoids play many important roles in biochemical functions such as respiration, photosynthesis, and regulation of growth in plants.  
Here, we explore the connection between two isoprenoid biosynthesis pathways and some downstream pathways. 

Similar to \citet{she2017robust}, we treat the $p=39$ genes from  two isoprenoid biosynthesis pathways as the predictors, and treat the $q=795$ genes from 56 downstream pathways as the response. 
Thus, $\Xb \in \RR^{118\times 39}$ and $\Yb \in \RR^{118\times 795}$, and we are interested in fitting the model $\Yb = \Xb \Ab+ \Eb$.  
We scale each element of $\Xb$ such that $\max_{i,j} |X_{ij}| = 1$, and standardize each column of $\Yb$ to have mean zero and standard deviation one.  
To assess whether there are outliers in $\Yb$, we perform Grubbs' test on each column of $\Yb$ \citep{grubbs}.  
Grubbs' test, also known as the maximum normalized residual test, is used to detect outliers from a normal distribution. 
After a Bonferroni correction, we find that 260 genes contain outliers.
In Figure~\ref{fig:hist}, we plot histograms for three genes that contain outliers.

\begin{figure}[!htp]
\centering
{\includegraphics[scale=0.52]{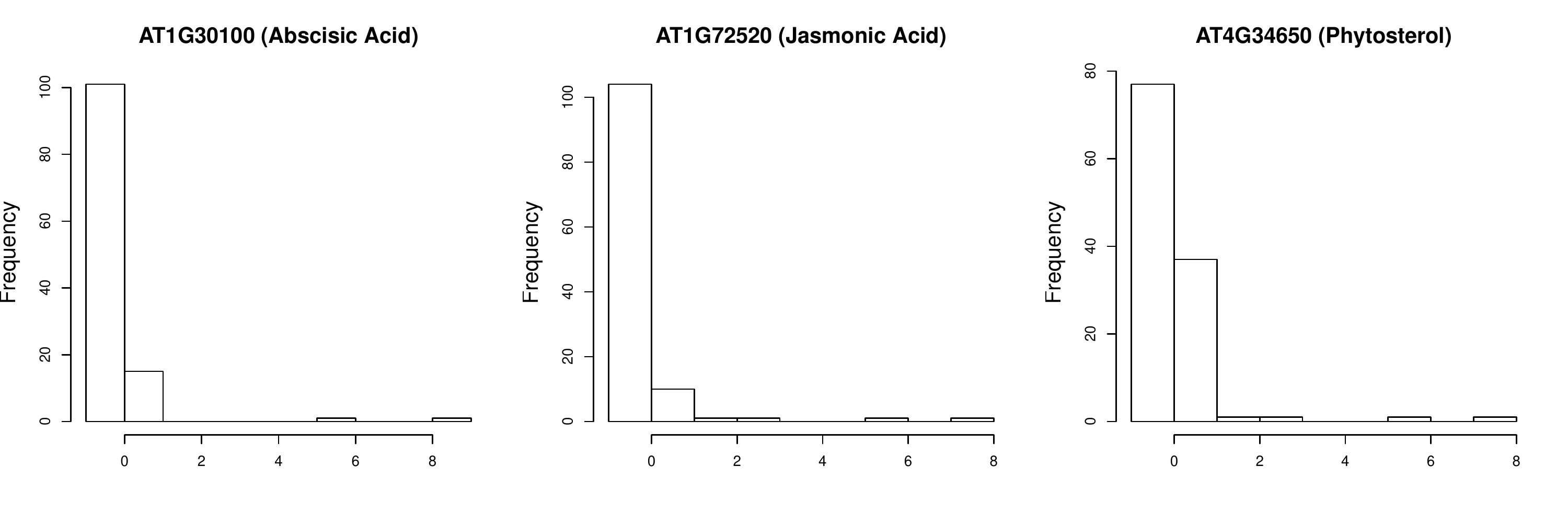}}
    \caption{Histograms for three genes from the abscisic acid, jasmonic acid, and phytosterol  pathways that are heavy-tailed. These genes are AT1G30100, AT1G72520, and AT4G34650, respectively.       }
          \label{fig:hist}
\end{figure}

In Section~\ref{sim:matrix}, we illustrated with numerical studies that if the response variables are heavy-tailed, sparse reduced rank regression with squared error loss will lead to incorrect estimates.
We now illustrate the difference between solving~\eqref{eq:opt:conv} with Huber loss and squared error loss. 
We set $\gamma=3$, and pick $\lambda$ such that there are 1000 non-zeros in the estimated coefficient matrix.  
For the robust method, we set the robustification parameter to equal $\tau=3$ for simplicity. In principle, this quantity can be chosen using cross-validation.

Let $\hat{\Ab}_{\texttt{{hubersrrr}}}$ and $\hat{\Ab}_{\texttt{{srrr}}}$ be the estimated regression coefficients for the robust and non-robust methods, respectively. 
To measure the difference between the two approaches in terms of regression coefficients and prediction, we compute the quantities  $\|\hat{\Ab}_{\texttt{hubersrrrst}}-\hat{\Ab}_{\texttt{srrr}}\|_{\rF}/\|\hat{\Ab}_{\texttt{hubersrrr}}\|_{\rF} \approx 37\%$ and $\|\Xb\hat{\Ab}_{\texttt{hubersrrr}}-\Xb\hat{\Ab}_{\texttt{srrr}}\|_{\rF}/\|\Xb\hat{\Ab}_{\texttt{hubersrrr}}\|_{\rF} \approx 35\%$.

Figure~\ref{fig:scat} displays scatterplots of the right singular vectors of $\Xb \hat{\Ab}_{\texttt{srrr}}$ against the right singular vectors of $\Xb\hat{\Ab}_{\texttt{hubersrrr}}$. 
We see that while the first singular vectors are similar between the two methods, the second and third singular vectors are very different.
These results suggest that the regression coefficients and model predictions can be quite different between robust and non-robust methods when there are outliers, and that care needs to be taken during model fitting.

\begin{figure}[!htp]
\centering
{\includegraphics[scale=0.52]{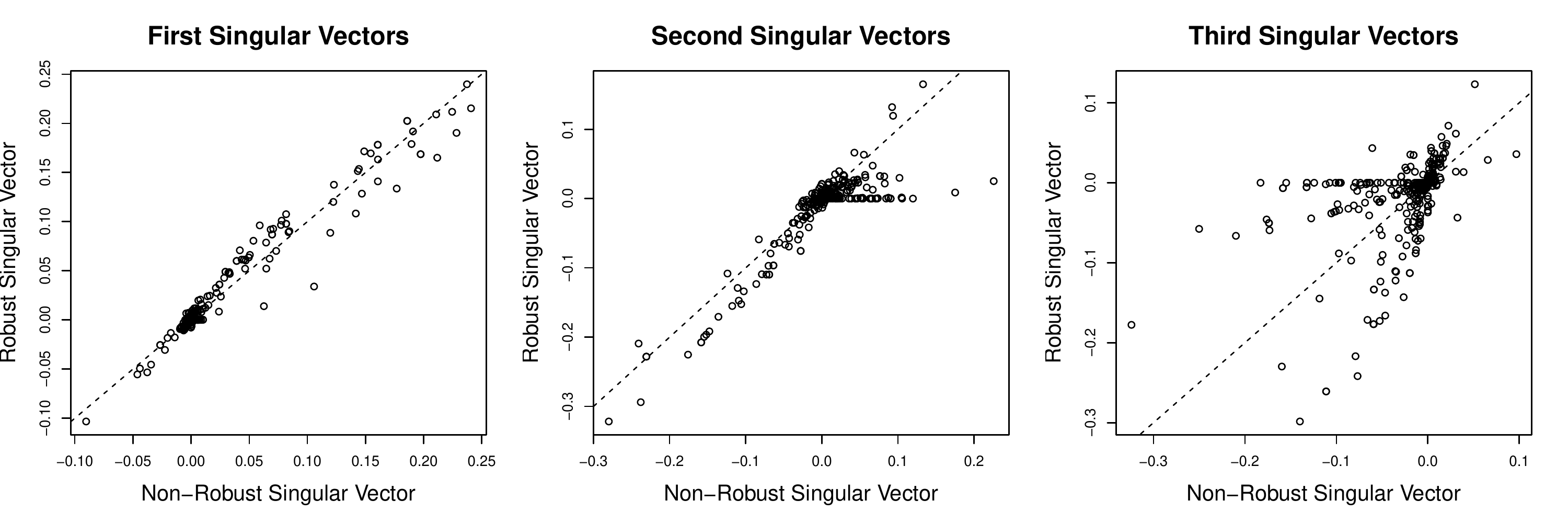}}
    \caption{Scatterplots of the leading right singular vectors of $\Xb \hat{\Ab}_{\texttt{hubersrrr}}$ and $\Xb \hat{\Ab}_{\texttt{srrr}}$.        }
          \label{fig:scat}
\end{figure}

\section{Discussion}
\label{sec:discuss}
We propose robust sparse reduced rank regression for analyzing large, complex, and possibly contaminated data. Our proposal is based on a convex relaxation, and is thus computationally tractable. 
We show that our proposal is statistically consistent under both Frobenius and nuclear norms in the high-dimensional setting in which $p>n$.      
By contrast, most of the existing literature in reduced rank regression focus on prediction and rank selection consistency.

In this paper, we focus on \emph{tail robustness}, i.e., the performance of an estimator in the presence of heavy-tailed noise.  
We show that the proposed robust estimator can achieve exponential-type deviation errors only under bounded low-order moments.
Tail robustness is different from the classical definition of robustness, which is  characterized by the breakdown point \citep{hampel1971general}, i.e., the proportion of outliers that a procedure can tolerate before it produces arbitrarily large estimates.
 However, the breakdown point does not shed light on the convergence properties of an estimator, such as consistency and efficiency.  
Intuitively, the breakdown point characterizes a form of the worst-case robustness, while  tail robustness corresponds to the average-case robustness. So a natural question arises: 
\begin{quote}
{\it 
What is the connection between the average-case robustness and the worst-case robustness? }
\end{quote}
We leave this for future work. 
 
 \bibliographystyle{ims}
\bibliography{reference}

\begin{thebibliography}{42}
\expandafter\ifx\csname natexlab\endcsname\relax\def\natexlab#1{#1}\fi
\expandafter\ifx\csname url\endcsname\relax
  \def\url#1{\texttt{#1}}\fi
\expandafter\ifx\csname urlprefix\endcsname\relax\def\urlprefix{URL }\fi

\bibitem[{Boyd et~al.(2010)Boyd, Parikh, Chu, Peleato and Eckstein}]{BoydADMM}
\textsc{Boyd, S.}, \textsc{Parikh, N.}, \textsc{Chu, E.}, \textsc{Peleato, B.}
  and \textsc{Eckstein, J.} (2010).
\newblock Distributed optimization and statistical learning via the {ADMM}.
\newblock \textit{Foundations and Trends in Machine Learning} \textbf{3}
  1--122.

\bibitem[{Bunea et~al.(2011)Bunea, She and Wegkamp}]{bunea2011optimal}
\textsc{Bunea, F.}, \textsc{She, Y.} and \textsc{Wegkamp, M.~H.} (2011).
\newblock Optimal selection of reduced rank estimators of high-dimensional
  matrices.
\newblock \textit{The Annals of Statistics} \textbf{39} 1282--1309.

\bibitem[{Bunea et~al.(2012)Bunea, She and Wegkamp}]{bunea2012joint}
\textsc{Bunea, F.}, \textsc{She, Y.} and \textsc{Wegkamp, M.~H.} (2012).
\newblock Joint variable and rank selection for parsimonious estimation of
  high-dimensional matrices.
\newblock \textit{The Annals of Statistics} \textbf{40} 2359--2388.

\bibitem[{Candes et~al.(2011)Candes, Li, Ma and Wright}]{robustpca}
\textsc{Candes, E.~J.}, \textsc{Li, X.}, \textsc{Ma, Y.} and \textsc{Wright,
  J.} (2011).
\newblock Robust principal component analysis?
\newblock \textit{Journal of ACM} \textbf{58} 1--37.

\bibitem[{Chandrasekaran et~al.(2012)Chandrasekaran, Parrilo and
  Willsky}]{willsky2012latent}
\textsc{Chandrasekaran, V.}, \textsc{Parrilo, P.~A.} and \textsc{Willsky,
  A.~S.} (2012).
\newblock Latent variable graphical model selection via convex optimization.
\newblock \textit{The Annals of Statistics} \textbf{40} 1935--1967.

\bibitem[{Chen et~al.(2012)Chen, Chan and Stenseth}]{chen2012reduced}
\textsc{Chen, K.}, \textsc{Chan, K.-S.} and \textsc{Stenseth, N.~C.} (2012).
\newblock Reduced rank stochastic regression with a sparse singular value
  decomposition.
\newblock \textit{Journal of the Royal Statistical Society: Series B
  (Statistical Methodology)} \textbf{74} 203--221.

\bibitem[{Chen et~al.(2013)Chen, Dong and Chan}]{chen2013reduced}
\textsc{Chen, K.}, \textsc{Dong, H.} and \textsc{Chan, K.-S.} (2013).
\newblock Reduced rank regression via adaptive nuclear norm penalization.
\newblock \textit{Biometrika} \textbf{100} 901--920.

\bibitem[{Chen and Huang(2012)}]{chen2012sparse}
\textsc{Chen, L.} and \textsc{Huang, J.~Z.} (2012).
\newblock Sparse reduced-rank regression for simultaneous dimension reduction
  and variable selection.
\newblock \textit{Journal of the American Statistical Association} \textbf{107}
  1533--1545.

\bibitem[{Chen et~al.(2018)Chen, Gao and Ren}]{chen2018robust}
\textsc{Chen, M.}, \textsc{Gao, C.} and \textsc{Ren, Z.} (2018).
\newblock Robust covariance and scatter matrix estimation under {H}uber's
  contamination model.
\newblock \textit{The Annals of Statistics} \textbf{46} 1932--1960.

\bibitem[{Cont(2001)}]{cont2001empirical}
\textsc{Cont, R.} (2001).
\newblock Empirical properties of asset returns: stylized facts and statistical
  issues.
\newblock \textit{Quantitive Finance} \textbf{1} 223--236.

\bibitem[{Eckstein and Bertsekas(1992)}]{EcksteinADMM92}
\textsc{Eckstein, J.} and \textsc{Bertsekas, D.} (1992).
\newblock On the {D}ouglas-{R}achford splitting method and the proximal point
  algorithm for maximal monotone operators.
\newblock \textit{Mathematical Programming} \textbf{55} 293--318.

\bibitem[{Fan et~al.(2018)Fan, Liu, Sun and Zhang}]{fan2015tac}
\textsc{Fan, J.}, \textsc{Liu, H.}, \textsc{Sun, Q.} and \textsc{Zhang, T.}
  (2018).
\newblock {I-LAMM}: for sparse learning: simultaneous control of algorithmic
  complexity and statistical error.
\newblock \textit{The Annals of Statistics} \textbf{46} 818--841.

\bibitem[{Grubbs(1950)}]{grubbs}
\textsc{Grubbs, F.~E.} (1950).
\newblock Sample criteria for testing outlying observations.
\newblock \textit{The Annals of Mathematical Statistics} \textbf{21} 27--58.

\bibitem[{Hampel(1971)}]{hampel1971general}
\textsc{Hampel, F.~R.} (1971).
\newblock A general qualitative definition of robustness.
\newblock \textit{The Annals of Mathematical Statistics}  1887--1896.

\bibitem[{He and Shao(1996)}]{he1996general}
\textsc{He, X.} and \textsc{Shao, Q.-M.} (1996).
\newblock A general {B}ahadur representation of {M}-estimators and its
  application to linear regression with nonstochastic designs.
\newblock \textit{The Annals of Statistics} \textbf{24} 2608--2630.

\bibitem[{Huber(1964)}]{huber1964robust}
\textsc{Huber, P.~J.} (1964).
\newblock Robust estimation of a location parameter.
\newblock \textit{The Annals of Mathematical Statistics} \textbf{35} 73--101.

\bibitem[{Huber(1973)}]{huber1973robust}
\textsc{Huber, P.~J.} (1973).
\newblock Robust regression: asymptotics, conjectures and {M}onte {C}arlo.
\newblock \textit{The Annals of Statistics} \textbf{1} 799--821.

\bibitem[{Izenman(1975)}]{izenman1975reduced}
\textsc{Izenman, A.~J.} (1975).
\newblock Reduced-rank regression for the multivariate linear model.
\newblock \textit{Journal of Multivariate Analysis} \textbf{5} 248--264.

\bibitem[{Jojic et~al.(2011)Jojic, Saria and Koller}]{jojic2011convex}
\textsc{Jojic, V.}, \textsc{Saria, S.} and \textsc{Koller, D.} (2011).
\newblock Convex envelopes of complexity controlling penalties: the case
  against premature envelopment.
\newblock In \textit{Proceedings of the Fourteenth International Conference on
  Artificial Intelligence and Statistics}.

\bibitem[{Liu et~al.(2012)Liu, Lin, Yan, Sun, Yu and Ma}]{rss2016}
\textsc{Liu, G.}, \textsc{Lin, Z.}, \textsc{Yan, S.}, \textsc{Sun, J.},
  \textsc{Yu, Y.} and \textsc{Ma, Y.} (2012).
\newblock Robust recovery of subspace structures by low-rank representation.
\newblock \textit{IEEE Transactions on Pattern Analysis and Machine
  Intelligence} \textbf{35} 171--184.

\bibitem[{Liu(1990)}]{liu1990notion}
\textsc{Liu, R.~Y.} (1990).
\newblock On a notion of data depth based on random simplices.
\newblock \textit{The Annals of Statistics}  405--414.

\bibitem[{Ma et~al.(2007)Ma, Gong and Bohnert}]{Maetal2007}
\textsc{Ma, S.}, \textsc{Gong, Q.} and \textsc{Bohnert, H.} (2007).
\newblock An {Arabidopsis} gene network based on the graphical {Gaussian}
  model.
\newblock \textit{Genome Research} \textbf{17} 1614--1625.

\bibitem[{Mammen(1989)}]{mammen1989asymptotics}
\textsc{Mammen, E.} (1989).
\newblock Asymptotics with increasing dimension for robust regression with
  applications to the bootstrap.
\newblock \textit{The Annals of Statistics} \textbf{17} 382--400.

\bibitem[{Mizera(2002)}]{mizera2002depth}
\textsc{Mizera, I.} (2002).
\newblock On depth and deep points: a calculus.
\newblock \textit{The Annals of Statistics} \textbf{30} 1681--1736.

\bibitem[{Mukherjee and Zhu(2011)}]{mukherjee2011reduced}
\textsc{Mukherjee, A.} and \textsc{Zhu, J.} (2011).
\newblock Reduced rank ridge regression and its kernel extensions.
\newblock \textit{Statistical analysis and data mining: the {ASA} data science
  journal} \textbf{4} 612--622.

\bibitem[{M{\"u}ller et~al.(1998)M{\"u}ller, Dacorogna and
  Pictet}]{muller1998heavy}
\textsc{M{\"u}ller, U.~A.}, \textsc{Dacorogna, M.~M.} and \textsc{Pictet,
  O.~V.} (1998).
\newblock Heavy tails in high-frequency financial data.
\newblock \textit{A Practical Guide to Heavy Tails: Statistical Techniques and
  Applications}  55--78.

\bibitem[{Portnoy(1985)}]{portnoy1985asymptotic}
\textsc{Portnoy, S.} (1985).
\newblock Asymptotic behavior of {M} estimators of $p$ regression parameters
  when $p^2/n$ is large; ii. normal approximation.
\newblock \textit{The Annals of Statistics} \textbf{13} 1403--1417.

\bibitem[{Portnoy and He(2000)}]{portnoy2000robust}
\textsc{Portnoy, S.} and \textsc{He, X.} (2000).
\newblock A robust journey in the new millennium.
\newblock \textit{Journal of the American Statistical Association} \textbf{95}
  1331--1335.

\bibitem[{Recht et~al.(2010)Recht, Fazel and Parrilo}]{MaryamNuclear}
\textsc{Recht, B.}, \textsc{Fazel, M.} and \textsc{Parrilo, P.} (2010).
\newblock Guaranteed minimum-rank solutions of linear matrix equations via
  nuclear norm minimization.
\newblock \textit{SIAM Review} \textbf{52} 471--501.

\bibitem[{Rodr\'{i}gues-Concepci\'{o}n and Boronat(2002)}]{Rodriguez02}
\textsc{Rodr\'{i}gues-Concepci\'{o}n, M.} and \textsc{Boronat, A.} (2002).
\newblock Elucidation of the methylerythritol phosphate pathway for isoprenoid
  biosynthesis in bacteria and plastids. {A} metabolic milestone achieved
  through genomics.
\newblock \textit{Plant Physiology} \textbf{130} 1079--1089.

\bibitem[{Rousseeuw and Yohai(1984)}]{rousseeuw1984robust}
\textsc{Rousseeuw, P.} and \textsc{Yohai, V.} (1984).
\newblock Robust regression by means of {S}-estimators.
\newblock In \textit{Robust and Nonlinear Time Series Analysis}. Springer,
  256--272.

\bibitem[{Rousseeuw(1984)}]{rousseeuw1984least}
\textsc{Rousseeuw, P.~J.} (1984).
\newblock Least median of squares regression.
\newblock \textit{Journal of the American Statistical Association} \textbf{79}
  871--880.

\bibitem[{Salibian-Barrera and Zamar(2002)}]{salibian2002bootrapping}
\textsc{Salibian-Barrera, M.} and \textsc{Zamar, R.~H.} (2002).
\newblock Bootrapping robust estimates of regression.
\newblock \textit{The Annals of Statistics} \textbf{30} 556--582.

\bibitem[{She and Chen(2017)}]{she2017robust}
\textsc{She, Y.} and \textsc{Chen, K.} (2017).
\newblock Robust reduced-rank regression.
\newblock \textit{Biometrika} \textbf{104} 633--647.

\bibitem[{She and Owen(2011)}]{she2011outlier}
\textsc{She, Y.} and \textsc{Owen, A.~B.} (2011).
\newblock Outlier detection using nonconvex penalized regression.
\newblock \textit{Journal of the American Statistical Association} \textbf{106}
  626--639.

\bibitem[{Sun et~al.(2018)Sun, Zhou and Fan}]{sun2016adaptive}
\textsc{Sun, Q.}, \textsc{Zhou, W.} and \textsc{Fan, J.} (2018).
\newblock Adaptive {H}uber regression.
\newblock \textit{Journal of the American Statistical Association, in press\!}
  .

\bibitem[{Tan et~al.(2015)Tan, Witten and Shojaie}]{tanetal2015}
\textsc{Tan, K.}, \textsc{Witten, D.} and \textsc{Shojaie, A.} (2015).
\newblock The cluster graphical lasso for improved estimation of {Gaussian}
  graphical models.
\newblock \textit{Computational Statistics and Data Analysis} \textbf{85}
  23--36.

\bibitem[{Tukey(1975)}]{tukey1975mathematics}
\textsc{Tukey, J.~W.} (1975).
\newblock Mathematics and the picturing of data.
\newblock In \textit{Proceedings of the International Congress of
  Mathematicians, Vancouver, 1975}, vol.~2.

\bibitem[{Wang et~al.(2016)Wang, Gong, Chang, Huang and Zhou}]{rcc2016}
\textsc{Wang, Q.}, \textsc{Gong, P.}, \textsc{Chang, S.}, \textsc{Huang, T.~S.}
  and \textsc{Zhou, J.} (2016).
\newblock Robust convex clustering analysis.
\newblock In \textit{IEEE 16th International Conference on Data Mining}.

\bibitem[{Wille et~al.(2004)Wille, Zimmermann, Vranov\'{a}, F\"{u}rholz, Laule,
  Bleuler, Hennig, Prel\'{i}c, Rohr, Thiele, Zitzler, Gruissem and
  B\"{u}hlmann}]{WilleEtAl2004}
\textsc{Wille, A.}, \textsc{Zimmermann, P.}, \textsc{Vranov\'{a}, E.},
  \textsc{F\"{u}rholz, A.}, \textsc{Laule, O.}, \textsc{Bleuler, S.},
  \textsc{Hennig, L.}, \textsc{Prel\'{i}c, A.}, \textsc{Rohr, P.},
  \textsc{Thiele, L.}, \textsc{Zitzler, E.}, \textsc{Gruissem, W.} and
  \textsc{B\"{u}hlmann, P.} (2004).
\newblock Sparse graphical {Gaussian} modeling of the isoprenoid gene network
  in {Arabidopsis} thaliana.
\newblock \textit{Genome Biology} \textbf{5} 1--13.

\bibitem[{Yohai(1987)}]{yohai1987high}
\textsc{Yohai, V.~J.} (1987).
\newblock High breakdown-point and high efficiency robust estimates for
  regression.
\newblock \textit{The Annals of Statistics}  642--656.

\bibitem[{Zuo and Serfling(2000)}]{zuo2000general}
\textsc{Zuo, Y.} and \textsc{Serfling, R.} (2000).
\newblock General notions of statistical depth function.
\newblock \textit{The Annals of Statistics}  461--482.

\end{thebibliography}

\newpage 
\appendix

\section{Derivation of Algorithm~\ref{Alg:general}}
\label{appendix:alg}
We derive the ADMM algorithm for solving (\ref{eq:opt:conv2}).
Recall that $\Bb = (\Bb_D,\Bb_Z,\Bb_W)^\T$, $\tilde{\Xb} = (\Xb,\Ib,\Ib)^\T$, and $\bOmega = (\Db,\Zb,\Wb)^\T$.
The scaled augmented Lagrangian for (\ref{eq:opt:conv2}) takes the form 
 \begin{equation}
 \label{Appendix:lagrangian}
\begin{split}
\cL_\rho(\Ab,\Db,\Wb,\Zb,\Bb) &= \frac{1}{n}  \ell_\tau(\Yb-\Db)    + \lambda \gamma \|\Zb \|_{1,1}+ \lambda  \|\Wb \|_*   +\frac{\rho}{2}\left\|\bOmega +\Bb - \tilde{\Xb} \Ab \right\|^2_\rF.\\
\end{split}
\end{equation}

\noindent The ADMM algorithm requires the following updates:\begin{enumerate}
\item $\Ab^{t+1} \leftarrow \underset{\Ab}{\text{argmin  }} \cL_\rho (\Ab,\Db^t,\Wb^t,\Zb^t,\Bb^t)$.
\item $\Db^{t+1} \leftarrow \underset{\Db}{\text{argmin  }} \cL_\rho (\Ab^{t+1},\Db,\Wb^t,\Zb^t,\Bb^t)$.
\item $\Wb^{t+1} \leftarrow \underset{\Wb}{\text{argmin  }} \cL_\rho (\Ab^{t+1},\Db^{t+1},\Wb,\Zb^t,\Bb^t)$.
\item $\Zb^{t+1} \leftarrow \underset{\Zb}{\text{argmin  }} \cL_\rho (\Ab^{t+1},\Db^{t+1},\Wb^{t+1},\Zb,\Bb^t)$.
\item $\Bb^{t+1} \leftarrow \Bb^{t} + \rho( {\Xb}(\Ab^{t+1}) -\bOmega^{t+1} )$.
\end{enumerate}

 \textbf{Update for $\Ab$:}
To obtain an update for $\Ab$, we solve the following optimization problem
\[
\underset{\Ab}{\mathrm{minimize}}~ \left\|  \bOmega + \Bb - \tilde{\Xb}\Ab     \right\|_{\rF}^2.
\]
Thus, we obtain $\hat{\Ab} = (\tilde{\Xb}^{\T} \tilde{\Xb})^{-1} \tilde{\Xb}^{\T}(\bOmega +\Bb)$.\\

 \textbf{Update for $\Zb$:}
To obtain an update for $\Zb$, we need to solve the following optimization problem
\[
\underset{\Zb}{\mathrm{minimize}}~\frac{1}{2} \left\|  \Zb -  (\Ab- \Bb_Z)     \right\|_{\rF}^2 +  \frac{\lambda  \gamma}{\rho} \|\Zb\|_{1,1}.
\]
Thus, we have $\hat{\Zb} = S(\Ab-\Bb_Z,\lambda \gamma /\rho)$, where $S$ denote the soft-thresholding operator, applied element-wise to a matrix, i.e.,  $S(A_{ij},b) = \text{sign}(A_{ij}) \max( |A_{ij}|-b, 0)$.\\
 
 \textbf{Update for $\Wb$:}
To obtain an update for $\Wb$, it amounts to solving 
\[
\underset{\Wb}{\mathrm{minimize}}~\frac{1}{2} \left\|  \Wb -  (\Ab- \Bb_W)     \right\|_{\rF}^2 +  \frac{\lambda}{\rho} \|\Wb\|_{*}.
\]
 Let $\Ab-\Bb_W = \sum_{j=1}^{\min\{p,q\}}\omega_j \ab_j  \mathbf{b}_j^\T$ be the singular value decomposition of $\Ab-\Bb_W$.  Then, we obtain $\hat{\Wb} = \sum_{j=1}^{\min\{p,q\}}   \max \left( \omega_j - \lambda/\rho  ,0   \right) \ab_j \mathbf{b}_j^\T $.\\

 \textbf{Update for $\Db$:}
We solve the following problem to obtain an update for $\Db$:
\[
\underset{\Db}{\mathrm{minimize}}~\frac{1}{n} \ell_\tau \left( \Yb-\Db\right) +\frac{\rho}{2}\left\|  \Db -  (\Xb\Ab - \Bb_D)     \right\|_{\rF}^2.
\]
For notational convenience, let $\Cb =\Xb\Ab - \Bb_D$.  We can solve the above problem element-wise: 
\[
\underset{D_{ij}}{\mathrm{minimize}}~\frac{1}{n} \ell_\tau \left( Y_{ij}-D_{ij}\right) +\frac{\rho}{2}(  D_{ij} -  C_{ij})^2. 
\]
Recall the Huber loss function from Definition~\ref{Huber.def} that there are two cases. 

First, we assume that $|Y_{ij}-D_{ij}| \le \tau$.  Then, the above optimization problem reduces to 
\[
\underset{D_{ij}}{\mathrm{minimize}}~\frac{1}{2n}  \left( Y_{ij}-D_{ij}\right)^2 +\frac{\rho}{2}(  D_{ij} -  C_{ij})^2. 
\]
Thus, we have $\hat{D}_{ij} = (Y_{ij}+n\rho C_{ij})/(1+n\rho)$.  Substituting this into the constraint $|Y_{ij}-D_{ij}| \le \tau$, we have $|[n\rho(Y_{ij}-C_{ij})]/(1+n\rho)| \le \tau $. Thus,  $\hat{D}_{ij} = (Y_{ij}+n\rho C_{ij})/(1+n\rho)$ if $|[n\rho(Y_{ij}-C_{ij})]/(1+n\rho)| \le \tau $.

Next, we assume that $|Y_{ij}-D_{ij}| > \tau$.  To obtain an estimate of $D_{ij}$ in this case, we solve
\[
\underset{D_{ij}}{\mathrm{minimize}}~ \frac{\tau}{n} |Y_{ij}-D_{ij}| +\frac{\rho}{2}(  D_{ij} -  C_{ij})^2. 
\]
Let $H_{ij} = Y_{ij}-D_{ij}$.  By a change of variable, we consider solving 
\[
\underset{H_{ij}}{\mathrm{minimize}}~ \frac{1}{2}(  Y_{ij} -  C_{ij}  - H_{ij})^2 +\frac{\tau}{n\rho} |H_{ij}|, 
\]
which yields the solution $\hat{H}_{ij} = S(Y_{ij}-C_{ij},\tau/(n\rho))$.  Thus, we have $\hat{D}_{ij} =Y_{ij} -S(Y_{ij}-C_{ij},\tau/(n\rho)) $.

\section{Proof of Lemma~\ref{lemma:localized restricted}}
\label{appendix:lemma:local}
\begin{proof}
The proposed Huber loss function can be written as 
\[
\cL_{\tau} (\Ab) =\frac{1}{n}\sum_{i=1}^n \sum_{k=1}^q \ell_\tau(Y_{ik} - \Xb_{i\cdot}^\T \Ab_{.k} ).
\]
Let 
\[
\Tb_{i\tau}= \mathrm{diag}\{1(|Y_{i1}-\Xb_{i\cdot}^\T \Ab_{\cdot 1}| \le \tau),\ldots, 1(|Y_{iq}-\Xb_{i\cdot}^\T \Ab_{\cdot q}| \le \tau)\}.
\]
  It can be shown that the Hessian takes the form 
\[
\Hb_{\tau}(\Ab) =  \frac{1}{n} \sum_{i=1}^n \Tb_{i\tau}\otimes \Xb_{i\cdot}\Xb_{i\cdot}^\T,
\]
where $\otimes$ is the kronecker product between two matrices.
For notational convenience, let 
\[
\tilde{\Tb}_{i\tau}= \mathrm{diag}\{1(|Y_{i1}-\Xb_{i\cdot}^\T \Ab_{\cdot 1}| > \tau),\ldots, 1(|Y_{iq}-\Xb_{i\cdot}^\T \Ab_{\cdot q}| > \tau)\}.
\]

Let $\tilde{\ub} = \mathrm{vec}({\Ub})$. For any $(\Ub,\Ab) \in \cC(m,\xi,\eta)$, we have 
\begin{equation}
\label{eq:lemma:2:proof:2}
\begin{split}
\tilde{\ub}^{\T}\Hb_{\tau}(\Ab)\tilde{\ub}&= \tilde{\ub}^{\T}\left(\frac{1}{n} \sum_{i=1}^n \Tb_{i\tau}\otimes \Xb_{i\cdot}\Xb_{i\cdot}^\T \right)\tilde{\ub}\\
&=\tilde{\ub}^{\T}\left(\frac{1}{n} \sum_{i=1}^n \Ib_q\otimes \Xb_{i\cdot}\Xb_{i\cdot}^\T \right)\tilde{\ub} - \tilde{\ub}^{\T}\left(\frac{1}{n} \sum_{i=1}^n \tilde{\Tb}_{i\tau}\otimes \Xb_{i\cdot}\Xb_{i\cdot}^\T \right)\tilde{\ub}\\
&= \| \tilde{\mathbf{S}}^{1/2} \tilde{\ub}\|_2^2 - \tilde{\ub}^{\T}\left(\frac{1}{n} \sum_{i=1}^n \tilde{\Tb}_{i\tau}\otimes \Xb_{i\cdot}\Xb_{i\cdot}^\T \right)\tilde{\ub},
\end{split}
\end{equation}
where $ \tilde{\mathbf{S}} = n^{-1} \sum_{i=1}^n \Ib_q\otimes \Xb_{i\cdot}\Xb_{i\cdot}^\T$.
We now obtain an upper bound for each element in $\tilde{\Tb}_{i\tau}$.  For  $1\le j\le q$, 
\#
\label{eq:lemma:2:proof:3}
1(|Y_{ij} - \Xb_{i\cdot}^\T \Ab_{\cdot j}| > \tau) 
&=1(|Y_{ij} - \Xb_{i\cdot}^\T \Ab^*_{\cdot j}+\Xb_{i\cdot}^\T \Ab^*_{\cdot j}-\Xb_{i\cdot}^\T \Ab_{\cdot j}| > \tau) \nn \\
& \le 1(|E_{ij}| > \tau/2) + 1(| \Xb_{i\cdot}^\T (\Ab_{\cdot j}^*-\Ab_{\cdot j})| > \tau/2). 
\#
Moreover, we have
\#
\label{eq:lemma:2:proof:4}
1(| \Xb_{i\cdot}^\T (\Ab_{\cdot j}^*-\Ab_{\cdot j})| > \tau/2) 
&=1\left(\frac{2}{\tau}| \Xb_{i\cdot}^\T (\Ab_{\cdot j}^*-\Ab_{\cdot j})| > 1\right) \nn \\
&\le \frac{2}{\tau}| \Xb_{i\cdot}^\T (\Ab_{\cdot j}^*-\Ab_{\cdot j})| \\
&\le \frac{2\eta}{\tau} \underset{1\le i\le n}{\max}\|{\Xb_{i\cdot}}\|_\infty \\
&\le \frac{2\eta}{\tau}.
\#
where the second inequality holds by Holder's inequality and the condition that $\|\Ab^*_{\cdot j}-\Ab_{\cdot j}\|_{1}\le \eta$.
Let $\ub_j$ be the $j$th column of $\Ub$.  Since, $\tilde{\Tb}_{i\tau}$ is a diagonal matrix, we obtain
\begin{equation}
\label{eq:lemma:2:proof:5}
\begin{split}
&\tilde{\ub}^{\T}\left(\frac{1}{n} \sum_{i=1}^n \tilde{\Tb}_{i\tau}\otimes \Xb_{i\cdot}\Xb_{i\cdot}^\T \right)\tilde{\ub}\\
&=   \sum_{j=1}^q \ub_j^\T \left(  \frac{1}{n}\sum_{i=1}^n \Xb_{i\cdot}\Xb_{i\cdot}^\T \cdot1(|Y_{ij} - \Xb_{i\cdot}^\T \Ab_{\cdot j}| > \tau)   \right)\ub_j\\
&\le   \sum_{j=1}^q \ub_j^\T \left(  \frac{1}{n}\sum_{i=1}^n \Xb_{i\cdot}\Xb_{i\cdot}^\T \cdot1(|E_{ij}| > \tau/2)   \right)\ub_j \\
&\qquad+  \sum_{j=1}^q \ub_j^\T \left(  \frac{1}{n}\sum_{i=1}^n \Xb_{i\cdot}\Xb_{i\cdot}^\T \cdot 1(| \Xb_{i\cdot}^\T (\Ab_{\cdot j}^*-\Ab_{\cdot j})| > \tau/2)   \right)\ub_j\\
&\le \frac{2\eta}{\tau} \|\tilde{\mathbf{S}}^{1/2} \tilde{\ub}\|_2^2
+\underset{1\le i \le n}{\max} \sum_{j=1}^q(\Xb_{i\cdot}^\T \ub_j)^2\cdot \underset{1\le j \le q}{\max}  \left( \frac{1}{n}\sum_{i=1}^n 1(|E_{ij}|>\tau/2)\right),
\end{split}
\end{equation}
where the first inequality holds by~(\ref{eq:lemma:2:proof:3}) and the last inequality holds by~(\ref{eq:lemma:2:proof:4}).

By Lemma~\ref{lemma:truncate concentration}, for any $1\le j\le q$ and $t>0$, we have 
\[
 \frac{1}{n}\sum_{i=1}^n 1(|E_{ij}|>\tau/2)\le (2/\tau)^{1+\delta} \nu_\delta + \sqrt{t/n}
\]
with probability at least $1-\exp(-2t)$. Moreover, for any $1\le i\le n$, we have 
\[
\sum_{j=1}^q|\Xb_{i\cdot}^\T \ub_j| \le \|\Xb_{i\cdot}^\T\|_{\infty} \|\tilde{\ub}\|_{1}
\le (1+\xi) \|\tilde{\ub}_{\cS}\|_{1} \le (1+\xi)  \sqrt{m}\|\tilde{\ub}_{\cS}\|_2.
\]
Thus, combining the above with (\ref{eq:lemma:2:proof:2}) and (\ref{eq:lemma:2:proof:5}), we have 
\[
\tilde{\ub}^{\T}\Hb_{\tau}(\Ab)\tilde{\ub}\ge   \| \tilde{\mathbf{S}}^{1/2} \tilde{\ub}\|_2^2 - \frac{2\eta}{\tau}\| \tilde{\mathbf{S}}^{1/2} \tilde{\ub}\|_2^2 -   (1+\xi)^2 m \left[ (2/\tau)^{1+\delta} \nu_\delta + \sqrt{t/n}\right].
\]
Consequently, picking $\tau \ge \min(8\eta, C (m\nu_\delta)^{1/(1+\delta)})$, $t=\log(pq)/2$, and $n > C' (m^2 \log(pq))$ for sufficiently large $C$ and $C'$, we have 
\[
\tilde{\ub}^{\T}\Hb_{\tau}(\Ab)\tilde{\ub}\ge   \frac{3}{4}\kappa_{\mathrm{lower}} - m (1+\xi)^2 \left[ (2/\tau)^{1+\delta} \nu_\delta + \sqrt{t/n}\right]\ge \frac{1}{2}\kappa_{\mathrm{lower}},
\] 
with probability at least $1-(pq)^{-1}$. 

The upper bound $\tilde{\ub}^\T \Hb_{\tau} (\Ab)\tilde{\ub} \le \kappa_{\mathrm{upper}}$ can be obtained similarly. 

\end{proof}

\section{Proof of Theorem~\ref{thm:up}}
\label{appendix:thm:up}
Recall from~(\ref{eq:opt:conv}) that the optimization problem takes the form 
\begin{equation}
\label{eq:model:appendix}
\underset{\Ab}{\mathrm{minimize}} \;\bigg\{\cL_{\tau} (\Ab)+\lambda \left( \|\Ab\|_*+\gamma\|\Ab\|_{1,1}\right)\bigg\},
\end{equation}
where we use the notation $\cL_{\tau} (\Ab) =n^{-1}\sum_{i=1}^n \sum_{k=1}^q \ell_\tau(Y_{ik} - \Xb_{i\cdot}^\T \Ab_{.k} )$ for convenience throughout the proof.
We start with stating some facts and notation.  

Let $\Ab \in \RR^{p\times q}$ be a rank $r$ matrix with singular value decomposition $\Ub\bLambda \Vb^\T$, where $\Ub \in\RR^{p \times r}$, $\Vb \in\RR^{q \times r}$, and $\bLambda \in \RR^{r\times r}$.  
The sub-differential of the nuclear norm is then given by (see, for instance, \citealp{MaryamNuclear})
\begin{equation}
\label{eq:subdiff:nuclear}
\partial \|\Ab\|_{*} = \left\{  
\Ub \Vb^\T + \Wb : \Wb\in \RR^{p\times q}, \Ub^\T \Wb = \mathbf{0}, \Wb \Vb = \mathbf{0}, \|\Wb\|_2 \le 1
\right\}.
\end{equation}
Let $\cF(r)= \{\Ab \in \RR^{p\times q} : \mathrm{rank} (\Ab) \le r \}$ be the algebraic variety of matrices with rank at most $r$.
Then, the tangent space at $\Ab$ with respect to $\cF(r)$ is given by 
\[
T(\Ab) = \left\{     
\Ub \Wb_1^\T+ \Wb_2 \Vb^\T : \Wb_1 \in \RR^{q\times r} , \Wb_2 \in \RR^{p\times r}
\right\},
\]
where $T(\Ab)$ can be interpreted as a subspace in $\RR^{p\times q}$ \citep{willsky2012latent}.
We now state a connection between the sub-differential of the nuclear norm and its tangent space.
Let $\cP_{T(\Ab)}$ denote the projection operator onto $T(\Ab)$.  Then, it can be shown that the following relationship holds
\[
\tilde{\Nb} \in \partial \|\Ab\|_* \qquad \mathrm{if~and~only~if} \qquad \cP_{T(\Ab)} (\tilde{\Nb}) = \Ub \Vb^\T,~~ \| \cP_{T(\Ab)^{\perp}} \tilde{\Nb}    \|_2 \le 1.
\]

In addition, we define several quantities that will be used in the proof.   For any convex loss function $\cL_{\tau}(\cdot)$, the Bregman divergence between $\hat{\Ab}$ and $\Ab^*$ is
\[
D_{\cL} (\hat{\Ab},\Ab^*)  =  \cL_{\tau}(\hat{\Ab}) -\cL_{\tau}(\Ab^*) -     \langle  
\nabla \cL_{\tau}(\Ab^*),\hat{\Ab}-\Ab^*\rangle\ge 0.
\]
We define the symmetric Bregman divergence as 
\begin{equation}
\label{eq:symmetric bregman}
D_{\cL}^s (\hat{\Ab},\Ab^*)  = D_{\cL} (\hat{\Ab},\Ab^*) +D_{\cL} (\Ab^*,\hat{\Ab}) = \langle  
\nabla \cL_{\tau}(\hat{\Ab})- \nabla \cL_{\tau}(\Ab^*),\hat{\Ab}-\Ab^*\rangle\ge 0
\end{equation}
The proof involves obtaining an upper bound and a lower bound for the symmetric Bregman divergence.  
To this end, we state some technical lemmas that will be used in the proof.

\begin{lemma}\label{lemma:0}
Assume that the covariates are standardized such that $\max_{i,j}| X_{ij}|= 1$ and that   $E_{ik}$ is such that $v_{\delta} = \EE (|E_{ik}|^{1+\delta})<\infty$. Pick $\tau \ge C_1  \{n v_\delta/\log (pq)\}^{\min \{1/2,1/(1+\delta)\}}$, 
we have 
\[
\left\|\nabla \cL_{\tau}(\Ab^*) \right\|_{\infty,\infty} \le C_2 v_{\delta}^{1/\min(1+\delta,2) } \left( \frac{\log (pq)}{n}\right)^{\min \{1/2,\delta/(1+\delta)\}},
\]
with probability at least $1-(pq)^{-1}$, where $C_1$ and $C_2$ are universal constants.
\end{lemma}

\begin{lemma}[\sf $\ell_{1,1}$-Cone Property]\label{lemma:1}
Assume that  $\|\nabla\cL_{\tau}(\Ab^*)\|_{\infty, \infty}\leq \lambda /2$.  Let $\hat{\Ab}$ be a solution to~\eqref{eq:opt:conv}.  We have $\hat{\Ab}$  falls in the  following $\ell_{1,1}$-cone
\$
\big\|(\hat{\Ab}-\Ab^*)_{\cS^c}\big\|_{1,1}\leq \frac{2\gamma+5}{2\gamma-5}\big\|(\hat{\Ab}-\Ab^*)_\cS\big\|_{1,1}.
\$
\end{lemma}

 Let $\cU$ be the linear space spanned by the columns of $\Ub$, and $\cV$ the linear space spanned by the columns of $\Vb$. We denote by $\cU^\perp$ and $\cV^\perp$ the orthogonal complements of $\cU$ and $\cV$, respectively. 
\begin{lemma}[\sf Nuclear Cone Property]
\label{lemma:nucone}
Assume that  $\|\nabla\cL_\tau(\Ab^*)\|_{\infty, \infty}\leq \lambda/2$ and $\gamma\geq 1/2$. We have 
\$
\big\|\cP_{\cU^\perp}(\widehat\Ab-\Ab^*)\cP_{\cV^\perp}\big\|_*\leq \big\|\cP_{\cU}(\widehat\Ab-\Ab^*)\cP_{\cV}\big\|_*+(\gamma+0.5) \big\|\big(\widehat\Ab-\Ab^*\big)_\cS\big\|_{1,1}.
\$ 
\end{lemma}

\begin{lemma}[\sf Restricted Strong Convexity]
\label{lemma:2}
Under the same conditions as in Lemma~\ref{lemma:localized restricted}, for matrices $(\Ab,\Ub) \in \cC(m,\xi,\eta)$, we have
\$
\cD_{\cL}^s(\Ab, \Ab^*)\geq \frac{\kappa_{\mathrm{lower}}}{2} \|\Ab-\Ab^*\|_\rF^2,
\$
with probability at least $1-(pq)^{-1}$.
\end{lemma}

To prove Theorem~\ref{thm:up}, we obtain upper and lower bounds for the symmetric Bregman divergence, respectively. 

\begin{proof}
\textbf{Upper bound under Frobenius norm:}
By the first order optimality condition of (\ref{eq:opt:conv}), there exists $\tilde{\Nb}\in \partial \|\hat{\Ab}\|_*$ and $\tilde{\bGamma} \in \partial \|\hat{\Ab}\|_{1,1}$ such that 
\begin{equation}
\label{eq:first order optimality}
\nabla \cL_{\tau}(\hat{\Ab}) + \lambda (\tilde{\Nb}+ \gamma \tilde{\bGamma}) = \mathbf{0}. 
\end{equation} 
Substituting~(\ref{eq:first order optimality}) into~\eqref{eq:symmetric bregman}, we have 
\begin{equation}
\label{thm:up:eq:step1}
\begin{split}
D_{\cL}^s (\hat{\Ab},\Ab^*)   &=  \langle  
-\lambda \tilde{\Nb} - \lambda \gamma \tilde{\bGamma}   - \nabla \cL_{\tau}(\Ab^*),\hat{\Ab}-\Ab^*\rangle\\
&=  \underbrace{ \langle \nabla \cL_{\tau}(\Ab^*), \Ab^* -\hat{\Ab}\rangle}_{\mathrm{I}_1}
+  \underbrace{\lambda \langle  
\tilde{\Nb} ,\Ab^*-\hat{\Ab}\rangle}_{\mathrm{I}_2}+ \underbrace{ \lambda \gamma \langle \tilde{\bGamma}   ,\Ab^*-\hat{\Ab}\rangle}_{\mathrm{I}_3}.
\end{split}
\end{equation}

Upper bound on $\mathrm{I}_1$: By the Holder's inequality, we have 
\begin{equation}
\label{thm:up:eq:step2}
\begin{split}
\mathrm{I}_1 &\le \|\nabla \cL_{\tau}(\Ab^*)  \|_{\infty,\infty} \|\hat{\Ab}-\Ab^*\|_{1,1}\\
&\le \frac{\lambda}{2} \|\hat{\Ab}-\Ab^*\|_{1,1}  \\
&= \frac{\lambda}{2} \left( \|(\hat{\Ab}-\Ab^*)_{\cS}\|_{1,1} +  \|(\hat{\Ab}-\Ab^*)_{\cS^c}\|_{1,1}\right)\\
&\le \frac{2\lambda\gamma}{2\gamma - 5}  \|(\hat{\Ab}-\Ab^*)_{\cS}\|_{1,1},
\end{split} 
\end{equation}
where the last inequality holds by Lemma~\ref{lemma:1}. \\

Upper bound on $\mathrm{I}_2$: By the Holder's inequality, we have 
\begin{equation}
\label{thm:up:eq:step3}
\begin{split}
\mathrm{I}_2 &\le \lambda \|\tilde{\Nb}\|_{\infty,\infty}  \|\hat{\Ab}-\Ab^*\|_{1,1}\\
&\le \lambda \|\tilde{\Nb}\|_{2}  \|\hat{\Ab}-\Ab^*\|_{1,1}\\
&\le 2\lambda \|\hat{\Ab}-\Ab^*\|_{1,1}\\
&\le \frac{8\lambda \gamma}{2\gamma -5}  \|(\hat{\Ab}-\Ab^*)_{\cS}\|_{1,1}, 
\end{split} 
\end{equation}
where the second inequality holds by the fact that $\|\tilde{\Nb}\|_2 \le 2$, and the  last inequality holds by Lemma~\ref{lemma:1}.\\

Upper bound on $\mathrm{I}_3$: Similarly, by Holder's inequality and using the fact that $\|\tilde{\bGamma}\|_{\infty,\infty} \le 1$, we obtain 
\begin{equation}
\label{thm:up:eq:step4}
\mathrm{I}_3 \le \lambda \gamma  \|\tilde{\bGamma}\|_{\infty,\infty}  \|\hat{\Ab}-\Ab^*\|_{1,1}
\le \lambda\gamma \|\hat{\Ab}-\Ab^*\|_{1,1}
\le \frac{4\lambda\gamma^2}{2\gamma -5} \|(\hat{\Ab}-\Ab^*)_{\cS}\|_{1,1},
\end{equation}
where the last inequality holds by Lemma~\ref{lemma:1}.\\

Thus, substituting (\ref{thm:up:eq:step2}), (\ref{thm:up:eq:step3}), and (\ref{thm:up:eq:step4}) into (\ref{thm:up:eq:step1}), we obtain 
\begin{equation}
\label{thm:up:eq:step5}
D_{\cL}^s (\hat{\Ab},\Ab^*)   \le \frac{4\gamma^2 + 10 \gamma}{2\gamma -5 } \lambda \|(\hat{\Ab}-\Ab^*)_{\cS}\|_{1,1}\le \frac{4\gamma^2 + 10 \gamma}{2\gamma -5 } \lambda \sqrt{s} \|(\hat{\Ab}-\Ab^*)_{\cS}\|_{\rF},
\end{equation}
where $s \le r s_u s_v$ is the sparsity parameter of $\Ab^*$, that is $s=|\supp(\Ab^*)|$. \\

Next, we employ Lemma~\ref{lemma:2} to obtain a lower bound for the symmetric Bregman divergence. Lemma~\ref{lemma:2} requires the matrix $\Ab \in \cC(m,\xi,\eta)$. 
To this end, we construct the matrix $\hat{\Ab}_\eta =  \Ab^* + \zeta(\hat{\Ab}-\Ab^*)$ such that $\|\hat{\Ab}_\eta-\hat{\Ab}^*\|_{1,1} \le \eta$ for some $\eta>0$.  If $\|\hat{\Ab}-\Ab^*\| < \eta$, we set $\zeta = 1$, so $\hat{\Ab}_{\eta}=\hat{\Ab}$.  Otherwise, we pick $\zeta \in (0,1)$ such that $\|\hat{\Ab}_{\eta}-\Ab^*\|_{1,1} = \eta$. By Lemma~\ref{lemma:1}, it can be shown that $\hat{\Ab}_{\eta}$ falls in an $\ell_1$-cone, and thus, $\hat{\Ab}_{\eta} \in \cC(m,\xi,\eta)$ with  
\begin{equation}
\label{thm:up:eq:step6-6}
\|(\hat{\Ab}_{\eta}-\Ab^*)_{\cS^c}\|_{1,1} \le \frac{2\gamma+5}{2\gamma-5} \|(\hat{\Ab}_{\eta}-\Ab^*)_{\cS}\|_{1,1} \qquad \mathrm{and} \qquad \|\hat{\Ab}_{\eta}-\Ab^*\|_{1,1} \le \eta.
\end{equation}
  
  Therefore, by Lemma~\ref{lemma:2}, we have 
\begin{equation}
\label{thm:up:eq:step6}
D_{\cL}^s (\hat{\Ab}_{\eta},\Ab^*) \ge \frac{\kappa_{\mathrm{lower}}}{2} \|\hat{\Ab}_{\eta}-\Ab^*\|_\rF^2.
\end{equation}
By Lemma~A.1 of \citet{sun2016adaptive},  
\begin{equation}
\label{thm:up:eq:step6-5}
D_{\cL}^s (\hat{\Ab}_{\eta},\Ab^*)   \le \zeta D_{\cL}^s (\hat{\Ab},\Ab^*).
\end{equation}
Combining \eqref{thm:up:eq:step6} and \eqref{thm:up:eq:step6-5} yields
\[
\|\hat{\Ab}_{\eta}-\Ab^*\|_{\rF}^2 \le \zeta  \kappa_{\mathrm{lower}}^{-1} \frac{8\gamma^2 + 20\gamma}{2\gamma -5} \lambda\sqrt{s}\| \hat{\Ab}-\Ab^*\|_{\rF}.
\]
Since $\hat{\Ab}-\Ab^* = \zeta^{-1} (\hat{\Ab}_{\eta}- \Ab^*)$, this yields 
\[
\|\hat{\Ab}_{\eta}-\Ab^*\|_{\rF} \le  \kappa_{\mathrm{lower}}^{-1} \frac{8\gamma^2 + 20\gamma}{2\gamma -5} \lambda\sqrt{s}.
\]
Finally, by \eqref{thm:up:eq:step6-6}, we have
\[
\|\hat{\Ab}_{\eta}-\Ab^*\|_{1,1}  \le \frac{4\gamma \sqrt{s}}{2\gamma-5} \|(\hat{\Ab}_{\eta}-\Ab^*)_{\cS}\|_{\rF}  \le \kappa_{\mathrm{lower}}^{-1} \frac{4\gamma}{2\gamma-5} \frac{8\gamma^2 + 20\gamma}{2\gamma -5} \lambda s < \eta,
\]
where the last inequality holds by the assumption that $n > C s^2 \log (pq)$ for some sufficiently large constant $C>0$. By the construction of $\hat{\Ab}_{\eta}$, since $\|\hat{\Ab}_{\eta}-\Ab^*\|_{1,1}   < \eta$, we have $\hat{\Ab}_{\eta}=\hat{\Ab}$, implying 
\[
\|\hat{\Ab}-\Ab^*\|_{\rF} \le  \kappa_{\mathrm{lower}}^{-1} \frac{8\gamma^2 + 20\gamma}{2\gamma -5} \lambda\sqrt{s}.\\
\]

\textbf{Upper bound under nuclear norm:}
Next, we establish an upper bound for $\hat{\Ab}-\Ab^*$ under the nuclear norm.  Recall that $s=|\supp(\Ab^*)|.$ We have shown previously that $\hat{\Ab}$ is in the local cone. 
 Applying Lemma \ref{lemma:nucone}, we can bound $\|\cP_{\cU^\perp}\big(\widehat \Ab-\Ab^*\big)\cP_{\cV^\perp}\|_*$ as 
\$
\big\|\cP_{U^\perp}\big(\widehat \Ab-\Ab^*\big)\cP_{V^\perp}\big\|_*
&\leq  \big\|\cP_{\cU}(\widehat\Ab-\Ab^*)\cP_{\cV}\big\|_*+(\gamma+0.5) \big\|\big(\widehat\Ab-\Ab^*\big)_\cS\big\|_{1,1}\\
&\leq \sqrt{r}\big\|\cP_{\cU}(\widehat\Ab-\Ab^*)\cP_{\cV}\big\|_\rF+(\gamma+0.5)\sqrt{s}\big\|\widehat \Ab-\Ab^*\big\|_{\rF}\\
&\lesssim  \kappa_{\mathrm{lower}}^{-1} \frac{4\gamma^2 + 10\gamma}{2\gamma -5} \lambda\sqrt{s}\big\{\sqrt{r}\vee (\gamma+0.5)\sqrt{s}\big\}.
\$

Thus, we have
\$
\big\|\widehat \Ab-\Ab^*\big\|_*
&\leq \big\|\cP_{T_*}\big(\widehat \Ab-\Ab^*\big)\big\|_*+\big\|\cP_{T_*^\perp}\big(\widehat \Ab-\Ab^*\big)\big\|_*\\
&\lesssim \kappa_{\mathrm{lower}}^{-1} \frac{4\gamma^2 + 10\gamma}{2\gamma -5} \lambda\sqrt{s}\big\{2\sqrt{r}\vee (\gamma+0.5)\sqrt{s}\big\}\\
&\leq C_\gamma\kappa^{-1}_{\mathrm{lower}}\lambda\sqrt{s}(\sqrt{r}\vee \sqrt{s})\\
&\lesssim\kappa^{-1}_{\mathrm{lower}}\lambda\sqrt{s}(\sqrt{r}\vee \sqrt{s}),
\$
where $C_\gamma=(2\gamma-5)^{-1}(4\gamma^2+10\gamma)\big\{2\vee (\gamma+0.5)\big\}$ is a constant depending only on $\gamma$.

\end{proof}
\section{Proof of Lemmas in Appendix~\ref{appendix:thm:up}}
\label{appendix:proof of lemmas}

\subsection{Proof of Lemma~\ref{lemma:0}}
\label{appendix:proof of lemma0}
\begin{proof}
To obtain an upper bound for $\|\nabla \cL_{\tau}(\Ab^*) \|_{\infty,\infty}$, 
we first obtain an upper bound for a single element of the gradient and then use a union bound argument to obtain an upper bound for the max norm.
Recall from~(\ref{eq:model:appendix}) that $\cL_{\tau} (\Ab^*) = \ell_\tau (\Yb - \Xb\Ab^*)/n$ and note that $E_{ik}= Y_{ik} - \Xb_{i\cdot}^\T \Ab_{\cdot k}^*$, where $\Xb_{i\cdot}$ and $\Ab_{\cdot k}^*$ are the $i$th row of $\Xb$ and $k$th column of $\Ab^*$, respectively.  
 Taking the gradient of $\cL_{\tau}(\Ab^*)$ with respect to $A_{jk}^*$, we obtain
\begin{equation}
\label{eq:lemma0:1}
\{\nabla \cL_{\tau} (\Ab^*)\}_{jk} = -\frac{1}{n} \sum_{i=1}^n 
X_{ij} \left\{ E_{ik} 1(|E_{ik}|\le \tau)
+ \tau 1(E_{ik}>\tau)- \tau 1(E_{ik}<-\tau) 
\right\}.
\end{equation}
It remains to obtain an upper bound for~(\ref{eq:lemma0:1}).  
To this end, we define the quantity
\[
\psi (u) = u 1(|u|\le 1) + 1(u>1) - 1(u<-1).
\]

We will consider two cases: (i) $0< \delta \le 1$ and (ii) $\delta >1$. 
When $0<\delta\le 1$, it can be verified that $\psi (u)$ has the following lower and upper bounds for all $u \in \RR$
\begin{equation}
\label{eq:lemma0:2}
- \log \left(1-u+|u|^{1+\delta}\right) \le \psi (u) \le \log \left(1+u+|u|^{1+\delta}\right).
\end{equation}
Using the notation $\psi(u)$, the gradient can be rewritten as 
\[
\{\nabla \cL_{\tau} (\Ab^*)\}_{jk} = -\frac{\tau}{n}\sum_{i=1}^n X_{ij}  \psi(E_{ik}/\tau).
\]  

Next, we obtain an upper bound for $X_{ij} \psi(E_{ik}/\tau)$. By~(\ref{eq:lemma0:2}), we have
\$
X_{ij} \psi(E_{ik}/\tau) 
&\le 1(X_{ij}\ge 0)  X_{ij} \log \left(1+E_{ik}/\tau+|E_{ik}/\tau|^{1+\delta}\right)\\
&~~~~ - 1(X_{ij}< 0)X_{ij}  \log \left(1-E_{ik}/\tau+|E_{ik}/\tau|^{1+\delta}\right).
\$
Since only one of the two terms on the upper bound is nonzero, we have
\begin{equation*}
\begin{split}
&\exp \{X_{ij} \psi(E_{ik}/\tau) \}\\
&\le  \left(1+E_{ik}/\tau+|E_{ik}/\tau|^{1+\delta}\right)^{1(X_{ij}\ge 0)  X_{ij}} +\left(1-E_{ik}/\tau+|E_{ik}/\tau|^{1+\delta}\right)^{-1(X_{ij}< 0)  X_{ij}} \\
&\le 1+\left( E_{ik}/\tau + |E_{ik}/\tau|^{1+\delta} \right) X_{ij},
\end{split}
\end{equation*}
where the last inequality follows from the inequality $(1+u)^v \le 1+uv$ for $u\ge -1$ and $0<v\le1$.
Using the above inequality, we obtain
\begin{equation}
\label{eq:lemma0:3}
\begin{split}
\EE  \left[\exp \left\{  \sum_{i=1}^n X_{ij} \psi (E_{ik}/\tau) \right\} \right]&= \prod_{i=1}^n\EE \left[ \exp \left\{X_{ij} \psi (E_{ik}/\tau) \right\}\right]\\
&\le \prod_{i=1}^n  \EE  \left[\left\{ 1+(E_{ik}/\tau)X_{ij}    + |E_{ik}/\tau|^{1+\delta}  X_{ij} \right\}\right]\\
&\le \prod_{i=1}^n  \EE  \left[\left\{ 1+  |E_{ik}/\tau|^{1+\delta}   \right\}\right]\\
&= \prod_{i=1}^n  \left\{ 1+  v_{\delta}/\tau^{1+\delta}   \right\}\\
&\le  \exp \left(nv_\delta / \tau^{1+\delta}\right),
\end{split}
\end{equation}
where the second inequality holds using the fact that $\EE[E_{ik}] = 0$ and that $\max_{i,j} |X_{ij}| = 1$, and the last inequality holds by the fact that $1+u \le \exp (u)$.  

Recall that $\{\nabla \cL_{\tau} (\Ab^*)\}_{jk} = -\tau n^{-1}\sum_{i=1}^n X_{ij}  \psi(E_{ik}/\tau)$.  
By the Markov's inequality and~(\ref{eq:lemma0:3}), for any $z>0$, we have
\begin{equation*}
\begin{split}
\PP \left(- \{\nabla \cL_{\tau} (\Ab^*)\}_{jk}   \ge v_{\delta} \tau z   \right) &=
\PP\left( \sum_{i=1}^n X_{ij} \psi (E_{ik}/\tau) \ge  n v_{\delta} z   \right) \\
&\le \frac{ \EE \left\{ \exp\left( \sum_{i=1}^n X_{ij} \psi (E_{ik}/\tau) \right)\right\}}{\exp (nv_{\delta} z)}  \\
&\le  \exp \left\{ -nv_{\delta} (z-\tau^{-(1+\delta)}) \right\}\\
&\le  \exp \left\{ -nv_{\delta}z/2  \right\},
\end{split}
\end{equation*}
where the last inequality holds by picking $\tau \ge (2/z)^{1/(1+\delta)}$. Similarly, it can be shown that $\PP \left( \{\nabla \cL_{\tau} (\Ab^*)\}_{jk}   \ge v_{\delta} \tau z   \right) \le \exp \left\{ -nv_{\delta}z/2  \right\}$.
  Then, by the union bound, we have
\begin{equation}
\label{eq:lemma0:4}
\begin{split}
\PP \left( \|\nabla \cL_{\tau} (\Ab^*)\|_{\infty,\infty}   \ge v_{\delta} \tau z  \right)&\le \sum_{j=1}^p \sum_{k=1}^q \PP \left( |\{\nabla \cL_{\tau} (\Ab^*)\}_{jk}|     \ge v_{\delta} \tau z  \right)\\
&\le 2 pq \exp (-nv_{\delta} z /2 ).
\end{split}
\end{equation}
Picking $z = (6/v_{\delta}) \log (pq)/n$ and $\tau \ge  \{(nv_{\delta})/(3\log (pq))\}^{1/(1+\delta)}  $ , we obtain
\[
\PP \left( \|\nabla \cL_{\tau} (\Ab^*)\|_{\infty,\infty}   \ge v_{\delta} \tau z  \right) \le \frac{1}{pq},
\]
 implying 
 \[
 \|\nabla \cL_{\tau} (\Ab^*)\|_{\infty,\infty}  \le 6^{\delta/(1+\delta)} (2v_{\delta})^{1/(1+\delta)} \left( \frac{\log (pq)}{n}\right)^{\delta/(1+\delta)}
 \]  
with probability at least $1-(pq)^{-1}$.

For $\delta >1$, instead of the inequality in~\eqref{eq:lemma0:2}, we use 
\[
- \log \left(1-u+|u|^{2}\right) \le \psi (u) \le \log \left(1+u+|u|^{2}\right).
\]
Following a similar argument, we arrive at
 \[
 \|\nabla \cL_{\tau} (\Ab^*)\|_{\infty,\infty}  \le 12^{1/2} v_{\delta}^{1/2} \left( \frac{\log (pq)}{n}\right)^{1/2}
 \]  
with probability at least $1-(pq)^{-1}$.  We obtain the desired results by combining both cases when $0<\delta\le 1$ and $\delta >1$.

\end{proof}

\subsection{Proof of Lemma~\ref{lemma:1}}
\label{appendix:proof of lemma1}
\begin{proof}
Recall that $\cS$ is the support of $\Ab^*$.  Under the condition that $\|\nabla \cL_{\tau}(\Ab^*)\|_{\infty,\infty} \le \lambda/2$, we will show that 
\[
\big\|(\hat{\Ab}-\Ab^*)_{\cS^c}\big\|_{1,1}\leq \frac{2\gamma+5}{2\gamma-5}\big\|(\hat{\Ab}-\Ab^*)_\cS\big\|_{1,1}.
\]

By the first order optimality condition of (\ref{eq:opt:conv}), there exists $\tilde{\Nb}\in \partial \|\hat{\Ab}\|_*$ and $\tilde{\bGamma} \in \partial \|\hat{\Ab}\|_{1,1}$ such that 
\begin{equation}
\label{eq:prooflemma1:1}
\langle \nabla \cL_{\tau}(\hat{\Ab}) + \lambda (\tilde{\Nb}+ \gamma \tilde{\bGamma}), \hat{\Ab}-\Ab^* \rangle=  0. 
\end{equation}
From (\ref{eq:symmetric bregman}), we have  $D_{\cL}^s (\hat{\Ab},\Ab^*)   = \langle  
\nabla \cL_{\tau}(\hat{\Ab})- \nabla \cL_{\tau}(\Ab^*),\hat{\Ab}-\Ab^*\rangle\ge 0$, implying
\begin{equation}
\label{eq:prooflemma1:2}
\langle  
\nabla \cL_{\tau}(\hat{\Ab}),\hat{\Ab}-\Ab^*\rangle  \ge \langle  \nabla \cL_{\tau}(\Ab^*),\hat{\Ab}-\Ab^*\rangle. 
\end{equation}
Substituting~(\ref{eq:prooflemma1:2}) into~(\ref{eq:prooflemma1:1}), we obtain 
\[
\langle \nabla \cL_{\tau}(\Ab^*) + \lambda (\tilde{\Nb}+ \gamma \tilde{\bGamma}), \hat{\Ab}-\Ab^* \rangle \le  0,
\]
or equivalently, 
\begin{equation}
\label{eq:prooflemma1:3}
\underbrace{\langle \nabla \cL_{\tau}(\Ab^*), \hat{\Ab}-\Ab^* \rangle}_{\mathrm{I}_1} + \underbrace{\lambda \langle \tilde{\Nb}, \hat{\Ab}-\Ab^* \rangle}_{\mathrm{I}_2} +   \underbrace{\lambda \gamma \langle\tilde{\bGamma}, \hat{\Ab}-\Ab^* \rangle}_{\mathrm{I}_3} \le0,
\end{equation}
It remains to obtain lower bounds for $\mathrm{I}_1,\mathrm{I}_2$, and  $\mathrm{I}_3$.\\

Lower bound for $\mathrm{I}_1$: By the Holder's inequality and the condition that $\|\nabla \cL_{\tau}(\Ab^*)\|_{\infty,\infty}\le \lambda/2$, we can lower bound $\mathrm{I}_1$ by
\begin{equation}
\label{eq:prooflemma1:4}
\mathrm{I}_1\ge  - \|\nabla \cL_{\tau}(\Ab^*)\|_{\infty,\infty} \|\hat{\Ab}-\Ab^* \|_{1,1}\ge -(\lambda/2) \|\hat{\Ab}-\Ab^* \|_{1,1}.
\end{equation}

Lower bound for $\mathrm{I}_2$: Similarly, by the Holder's inequality, we have 
\begin{equation}
\label{eq:prooflemma1:5}
\mathrm{I}_2\ge  - \lambda \|\tilde{\Nb}\|_{\infty,\infty} \|\hat{\Ab}-\Ab^* \|_{1,1}\ge -\lambda\|\tilde{\Nb}\|_{2}  \|\hat{\Ab}-\Ab^* \|_{1,1} \ge-2 \lambda  \|\hat{\Ab}-\Ab^* \|_{1,1},
\end{equation}
were the second inequality holds using the fact that $\|\tilde{\Nb}\|_{\infty,\infty} \le \|\tilde{\Nb}\|_{2}$ and the last inequality holds by $\|\tilde{\Nb}\|_2 \le 2$.


Lower bound for $\mathrm{I}_3$: By the definition of the subgradient of an $\ell_1$ norm, we have $\langle \tilde{\bGamma}, \hat{\Ab}  \rangle = \|\hat{\Ab}\|_{1,1}$ and that $\|\tilde{\bGamma}\|_{\infty,\infty}\le 1$.  Thus, we have 
\begin{equation}
\label{eq:prooflemma1:6}
\begin{split}
\mathrm{I}_3  &= \lambda\gamma \langle \tilde{\bGamma}_{\cS}, (\hat{\Ab}-\Ab^*)_{\cS}   \rangle
+ \lambda\gamma \langle \tilde{\bGamma}_{\cS^c}, (\hat{\Ab}-\Ab^*)_{\cS^c}   \rangle\\
&\ge -\lambda \gamma\|(\hat{\Ab}-\Ab^*)_{\cS}\|_{1,1} + \lambda\gamma \langle \tilde{\bGamma}_{\cS^c}, (\hat{\Ab}-\Ab^*)_{\cS^c}   \rangle\\
&\ge -\lambda \gamma\|(\hat{\Ab}-\Ab^*)_{\cS}\|_{1,1} + \lambda\gamma \|(\hat{\Ab}-\Ab^*)_{\cS^c}  \|_{1,1},
\end{split}
\end{equation}  
where the second inequality follows from Holder's inequality and the last inequality follows from the fact that $\langle \tilde{\bGamma}_{\cS^c},\hat{\Ab}_{\cS^c}  \rangle = \|\hat{\Ab}_{\cS^c}\|_{1,1}$ and that $\Ab^*_{\cS^c} = \mathbf{0}$.\\

Substituting~(\ref{eq:prooflemma1:4}),~(\ref{eq:prooflemma1:5}), and~(\ref{eq:prooflemma1:6}) into~(\ref{eq:prooflemma1:3}), we obtain
\[
-(\lambda/2) \|\hat{\Ab}-\Ab^*\|_{1,1} -2\lambda \|\hat{\Ab}-\Ab^*\|_{1,1}-\lambda \gamma\|(\hat{\Ab}-\Ab^*)_{\cS}\|_{1,1} + \lambda\gamma \|(\hat{\Ab}-\Ab^*)_{\cS^c}  \|_{1,1}\le 0.
\]
After rearranging the terms, we have
\[
 \|(\hat{\Ab}-\Ab^*)_{\cS^c}\|_{1,1} \le \frac{2\gamma+5}{2\gamma-5}  \|(\hat{\Ab}-\Ab^*)_{\cS}\|_{1,1}.
\]
\end{proof}

\subsection{Proof of Lemma~\ref{lemma:nucone}}
\label{appendix:proof of lemma nuclear}
\begin{proof}
From \eqref{eq:prooflemma1:1}--\eqref{eq:prooflemma1:4} in the proof of Lemma~\ref{lemma:1}, there exists $\tilde{\Nb} \in \partial \|\hat{\Ab}\|_*$ and $\tilde{\bGamma} \in \partial \|\hat{\Ab}\|_{1,1}$ such that  
\[
{\langle \nabla \cL_{\tau}(\Ab^*), \hat{\Ab}-\Ab^* \rangle}+ {\lambda \langle \tilde{\Nb}, \hat{\Ab}-\Ab^* \rangle}+{\lambda \gamma \langle\tilde{\bGamma}, \hat{\Ab}-\Ab^* \rangle} \le0.
\]
Moreover, by monotonicity of subdifferentials of convex functions, $\langle-\lambda(\widetilde \Nb- \Nb), \widehat \Ab -\Ab^*  \rangle\leq 0$, where $\Nb\in \partial \|\Ab^*\|_*$. Combining the above inequalities, we have
\begin{equation}
\label{eq:lemma:nuclear1}
\underbrace{ \lambda\langle  \Nb, \widehat\Ab-\Ab^*\rangle}_{\Rom{2}_1} +\underbrace{\lambda\gamma\langle \widetilde\bGamma, \widehat\Ab-\Ab^*\rangle}_{\Rom{2}_2} +\underbrace{\langle\nabla\cL(\Ab^*) , \widehat\Ab-\Ab^*\rangle}_{\Rom{2}_3}\leq 0. 
\end{equation}

Lower bound for $\Rom{2}_1$: Recall the sub-differential of the nuclear norm in (\ref{eq:subdiff:nuclear}). From \eqref{eq:subdiff:nuclear}, the subdifferential $\Nb$ can be written as 
\$
\Nb=\Ub \Vb^\T +\cP_{\cU^\perp}\Wb\cP_{\cV^\perp},~\textnormal{where}~\|\Wb\|_2\leq 1. 
\$
We choose $\Wb$ such that $\langle \cP_{\cU^\perp}\Wb\cP_{\cV^\perp}, \widehat \Ab-\Ab^* \rangle =\|\cP_{\cU^\perp}\widehat\Ab\cP_{\cV^\perp}\|_*$, and this  implies that 
\$
\Rom{2}_1&=\lambda\big\langle \Ub\Vb^\T+\cP_{\cU^\perp}\Wb\cP_{\cV^\perp}, \widehat\Ab-\Ab^* \big\rangle\\
&=\lambda\big\langle \Ub\Vb^\T, \cP_{\cU}(\widehat\Ab-\Ab^*)\cP_{\cV} \big\rangle +\lambda\big\langle \cP_{\cU^\perp}\Wb\cP_{\cV^\perp},  \widehat\Ab \big\rangle\\
&\geq \lambda\big\|\cP_{\cU^\perp}\widehat\Ab\cP_{\cV^\perp}\big\|_*-\lambda\big\|\cP_{\cU}(\widehat\Ab-\Ab^*)\cP_{\cV}\big\|_*. 
\$

Lower bound for $\Rom{2}_2$: using a similar argument to the proof of Lemma \ref{lemma:1}, we have 
\$
\Rom{2}_2  
&\ge -\lambda \gamma\|(\hat{\Ab}-\Ab^*)_{\cS}\|_{1,1} + \lambda\gamma \|(\hat{\Ab}-\Ab^*)_{\cS^c}  \|_{1,1}. 
\$ 

Lower bound for $\Rom{2}_3$: using a similar argument to the proof of Lemma \ref{lemma:1},  we obtain that 
\[
\Rom{2}_3\geq -\big\|\nabla\cL_\tau(\Ab^*)\big\|_{\infty, \infty}\big\|\widehat\Ab-\Ab^*\big\|_{1,1} \ge -\frac{\lambda}{2}\big\|\widehat\Ab-\Ab^*\big\|_{1,1}.
\] 

Therefore, combining the lower bounds for $\Rom{2}_1$, $\Rom{2}_2$ and $\Rom{2}_3$ into \eqref{eq:lemma:nuclear1}, we obtain
\$
&\lambda\big\|\cP_{\cU^\perp}(\widehat\Ab-\Ab^*)\cP_{\cV^\perp}\big\|_*-\lambda\big\|\cP_{\cU}(\widehat\Ab-\Ab^*)\cP_{\cV}\big\|_*-\lambda\gamma \big\|\big(\widehat\Ab-\Ab^*\big)_\cS\big\|_{1,1}\\
&~~~+\lambda\gamma\big\|\big(\widehat\Ab-\Ab^*\big)_{\cS^c}\big\|-(\lambda/2) \|\widehat\Ab-\Ab^*\|_{1,1}\leq 0.
\$
By the assumption that $\gamma\geq 1/2$, the above equation simplifies to 
\$
\big\|\cP_{\cU^\perp}(\widehat\Ab-\Ab^*)\cP_{\cV^\perp}\big\|_*\leq \big\|\cP_{\cU}(\widehat\Ab-\Ab^*)\cP_{\cV}\big\|_*+(\gamma+0.5) \big\|\big(\widehat\Ab-\Ab^*\big)_\cS\big\|_{1,1}.
\$

\end{proof}

\subsection{Proof of Lemma~\ref{lemma:2}}
\label{appendix:proof of lemma2}
\begin{proof}
Recall that 
\[
D_{\cL}^s (\Ab,\Ab^*)  =  \langle  
\nabla \cL_{\tau}({\Ab})- \nabla \cL_{\tau}(\Ab^*),{\Ab}-\Ab^*\rangle.  
\]
Let $\bDelta=\Ab -\Ab^*$.  By the mean value theorem, we have 
\[
D_{\cL}^s (\Ab,\Ab^*)   = \mathrm{vec}(\bDelta)^\T  \Hb_{\tau}  (\tilde{\Ab})  \mathrm{vec}(\bDelta),
\]
where $\tilde{\Ab}$ lies between $\Ab^*$ and $\Ab^*+\bDelta$. By Holder's inequality, we have
\[
D_{\cL}^s (\Ab,\Ab^*)   \ge \lambda_{\min} \left(\Hb_{\tau}  (\tilde{\Ab})\right)  \|\Ab-\Ab^*\|_\rF^2.
\]
It remains to show that $\lambda_{\min} (\Hb_{\tau}  (\tilde{\Ab}))$ is lower bounded by a constant.

Let $t\in [0,1]$.  Then, we can rewrite $\tilde{\mathbf{A}}$ as a convex combination of $\Ab^*$ and $\Ab^*+\bDelta$, i.e., $\tilde{\Ab} = t \Ab+ (1-t) \Ab^* $. Thus, by the triangle inequality, we have 
\[
\|\tilde{\Ab}-\Ab^*\|_{1,1} \le \|  t \Ab+ (1-t) \Ab^*  - \Ab^*\|_{1,1} \le t \|\Ab-\Ab^*\|_{1,1}\le t \eta.
\]  
Therefore, $\tilde{\Ab}\in \cC(m,\xi,\eta)$.   By Lemma~\ref{lemma:localized restricted}, we have $\lambda_{\min} (\Hb_{\tau}  (\tilde{\Ab})) \ge \kappa_{\mathrm{lower}}/2$ with probability $1-(pq)^{-1}$. Thus, 
\[
D_{\cL}^s (\Ab,\Ab^*)   \ge \frac{\kappa_{\mathrm{lower}}}{2}\|\Ab-\Ab^*\|_\rF^2.
\]

\end{proof}

\section{Technical Lemmas}
\label{appendix:tech}
\begin{lemma}[\sf Hoeffding's Inequality]\label{lemma:hoeffding}
Let $Z_1,\ldots,Z_n$ be independent random variables such that $\EE (Z_i)=\mu$ and $a\le Z_i \le b$.  Then, for any $z>0$,  
\[
\PP \left( \frac{1}{n}\sum_{i=1}^n Z_i \ge z+\mu \right) \le  \exp (-2nz^2/(b-a)^2).
\]
\end{lemma}

\begin{lemma}
\label{lemma:truncate concentration}
Let $X_1,\ldots,X_n$ be independent random variables with 
\[
\EE (X_i)= 0 \qquad \mathrm{and} \qquad v_{\delta} = \max_{i}~ \EE (|X_i|^{1+\delta}) <\infty~ \mathrm{for}~\delta>0.
\] 
For any $t\ge 0$ and $\tau>0$, we have 
\[
\PP \left( \frac{1}{n} \sum_{i=1}^n 1 (|X_i| > \tau/2 )   \ge  (2/\tau)^{1+\delta} v_{\delta} + \sqrt{t/n} \right) \le \exp (-2t).
\]

\end{lemma}
\begin{proof}
We first obtain an upper bound for $\EE (n^{-1} \sum_{i=1}^n1  (|X_i| > \tau/2 ) )$. By the Markov's inequality, we have 
\[
\EE \left(\frac{1}{n} \sum_{i=1}^n1  (|X_i| > \tau/2 ) \right) = \frac{1}{n}\sum_{i=1}^n \PP(|X_i| > \tau/2 ) = \frac{1}{n}\sum_{i=1}^n  \PP \left(|X_i|^{1+\delta} > (\tau/2)^{1+\delta} \right) \le (2/\tau)^{1+\delta} v_{\delta}.
\]
Let $Z_i = 1(|X_i| > \tau/2)$, $\mu = E(Z_i)$, and $z= \sqrt{t/n}$.  Note that $0 \le Z_i \le 1$.  
By Lemma~\ref{lemma:hoeffding}, we have 
\[
\PP \left( \frac{1}{n} \sum_{i=1}^n 1 (|X_i| > \tau/2 )   \ge  (2/\tau)^{1+\delta} v_{\delta} + \sqrt{t/n} \right) \le \exp (-2t),
\]
as desired.
\end{proof}

\end{document}